\documentclass[letterpaper]{article} 
\usepackage{aaai23}
\usepackage{times}  
\usepackage{helvet}  
\usepackage{courier}  
\usepackage[hyphens]{url}  
\usepackage{graphicx} 
\usepackage{natbib}  
\usepackage{caption} 
\usepackage{algorithm}
\usepackage{amsthm} 
\usepackage{algpseudocode}
\usepackage{booktabs}
\usepackage{multirow}
\usepackage{newfloat}
\usepackage{listings}
\DeclareCaptionStyle{ruled}{labelfont=normalfont,labelsep=colon,strut=off} 
\floatstyle{ruled}
\newfloat{listing}{tb}{lst}{}
\floatname{listing}{Listing}
\usepackage{amsmath,amssymb,mathrsfs}
\usepackage{graphicx}
\newtheorem{theorem}{Theorem}

\title{Incremental Permutation Feature Importance (iPFI): \\ Towards Online Explanations on Data Streams}

\author {
    Fabian Fumagalli,\textsuperscript{\rm 1,*}
    Maximilian Muschalik,\textsuperscript{\rm 2,*}
    Eyke Hüllermeier,\textsuperscript{\rm 2}
    Barbara Hammer\textsuperscript{\rm 1}
}
\affiliations {
    \textsuperscript{\rm 1} Bielefeld University, Bielefeld, Germany\\
    \textsuperscript{\rm 2} LMU Munich, Munich, Germany\\
    ffumagalli@techfak.uni-bielefeld.de, Maximilian.Muschalik@lmu.de, eyke@lmu.de, bhammer@techfak.uni-bielefeld.de
}

\begin{document}

\maketitle

\begin{abstract}
Explainable Artificial Intelligence (XAI) has mainly focused on static learning scenarios so far. We are interested in dynamic scenarios where data is sampled progressively, and learning is done in an incremental rather than a batch mode. We seek efficient incremental algorithms for computing feature importance (FI) measures, specifically, an incremental FI measure based on feature marginalization of absent features similar to permutation feature importance (PFI). We propose an efficient, model-agnostic algorithm called iPFI to estimate this measure incrementally and under dynamic modeling conditions including concept drift. We prove theoretical guarantees on the approximation quality in terms of expectation and variance. To validate our theoretical findings and the efficacy of our approaches compared to traditional batch PFI, we conduct multiple experimental studies on benchmark data with and without concept drift.
\end{abstract}

\section{Introduction}

Online learning from dynamic data streams is a prevalent machine learning (ML) approach for various application domains \cite{bahri_data_2021}.
For instance, predicting energy consumption for individual households can foster energy-saving strategies such as load-shifting.
Concept drift resulting from environmental changes, such as pandemic-induced lock-downs, drastically impacts the energy consumption patterns necessitating online ML \cite{GarciaMartin.2019}.
Explaining these predictions yields a greater understanding of an individual's energy use and enables prescriptive modeling for further energy-saving measures \cite{Wastensteiner.2021}.
For black-box machine learning methods, so-called post-hoc XAI methods seek to explain single predictions or entire models in terms of the contribution of specific features \cite{Adadi.2018}.
In this paper, we are interested in feature importance (FI) as a global assessment of features, which indicates their respective relevance to the given task and model.
A prominent representative of global FI is the permutation feature importance (PFI) \cite{breiman_random_2001}, which, in its original form, requires a holistic view of the entire dataset in a static batch learning environment.
More generally, explainable artificial intelligence (XAI) has been studied mainly in the batch setting, where learning algorithms operate on static datasets. 
In scenarios where data does not fit into memory or computation time is strictly limited, like in progressive data science for big datasets \cite{Turkay2018ProgressiveDS}, or rapid online learning from data streams \cite{bahri_data_2021}, this assumption prohibits the use of traditional FI or XAI measures.
Incremental, time- and memory-efficient implementations that provide anytime results have received much attention in recent years \cite{losing_incremental_2018,montiel_river_2020}.
In this article, we are interested in efficient incremental algorithms for FI.
Especially in the context of drifting data distributions, this task is particularly relevant --- but also challenging, as many common FI methods are already computationally costly in the batch setting.

\begin{figure}
    \centering
    \includegraphics[width=\columnwidth]{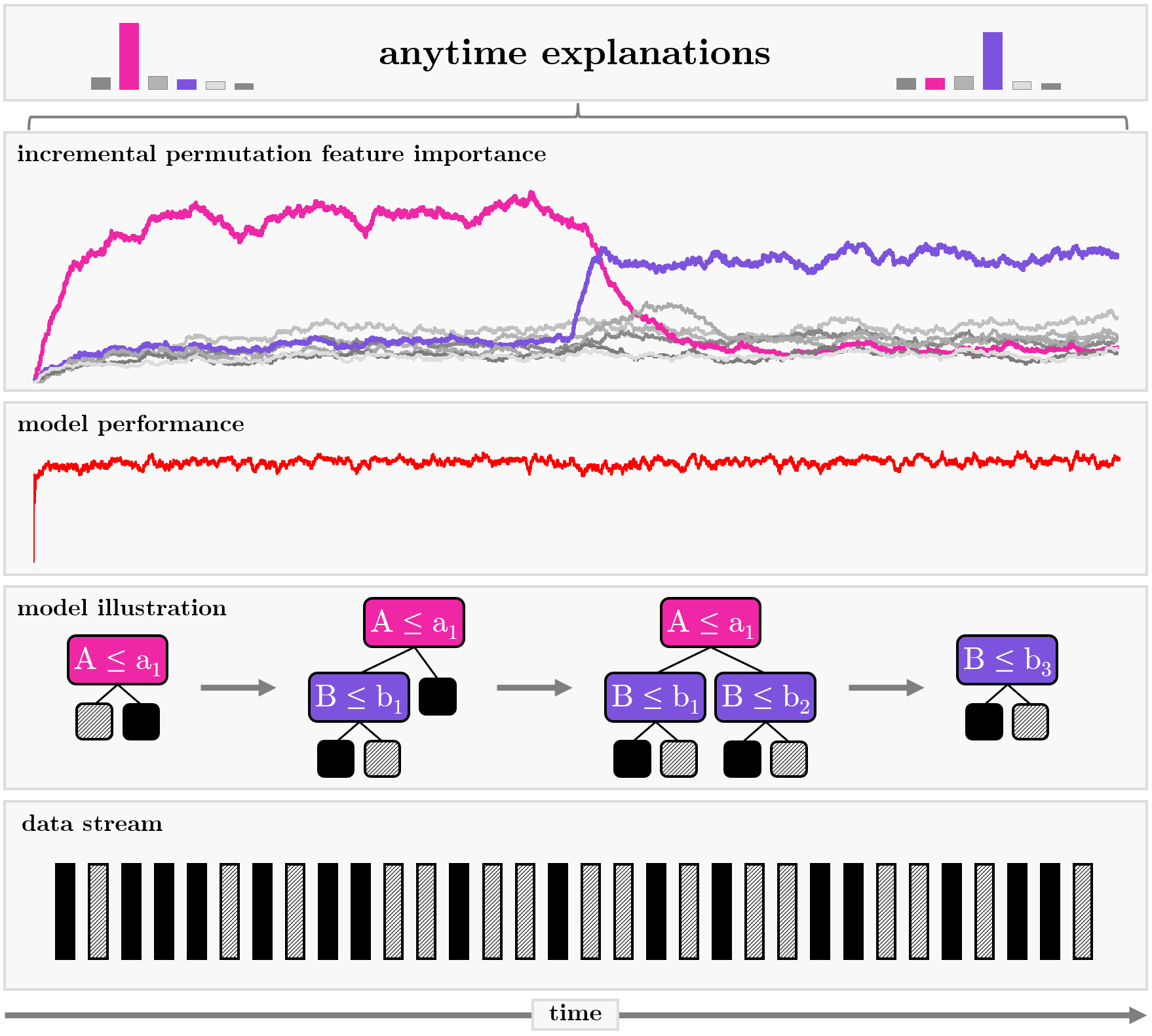}
    \caption{Incremental feature importance on an electricity-data stream to create anytime explanations. Concept drift in the data (rectangles) lead to model adaption without visible changes in the model's performance.}
    \label{fig:my_label}
\end{figure}

\paragraph{Contribution.}
Our core contributions include:
\begin{itemize}
    \item We establish the concrete connection of model reliance \cite{JMLR:v20:18-760} and permutation tests to conclude that only properly scaled permutation tests are unbiased estimates of global FI.

    \item We introduce an \emph{incremental estimator for PFI} (iPFI) with two sampling strategies to create marginal feature distributions in an incremental learning scenario.
    \item We provide theoretical guarantees regarding bias, variance, and approximation error, which can be controlled by a single sensitivity parameter, and analyze the estimation quality in the case of a static and an incrementally learned model.
    \item We implement iPFI and conduct experiments on its approximation quality compared to batch permutation tests, as well as its ability to efficiently provide anytime FI values under different types of concept drift.
\end{itemize}
All experiments and algorithms are publicly available and integrated into the well-known incremental learning framework \emph{river} \cite{montiel_river_2020}.

\paragraph{Related Work.}
A variety of model-agnostic local FI methods \cite{Ribeiro_2016,NIPS2017_7062,Lundberg2020,Covert.2021b} exist that provide relevance-values for single instances. 
In addition, model-specific variants have been proposed for neural networks \cite{bach_2015,selvaraju_2017}.
PFI and its extensions  \cite{molnar2020model,konig2021relative} are among global FI methods that provide relevance-values across all instances.

SAGE, a popular Shapley-based approach, has been proposed and compared with existing methods \cite{Covert_Lundberg_Lee_2020}.
As calculating FI values is computationally expensive, especially for Shapley-based methods, more efficient approaches such as FastSHAP \cite{jethani2021fastshap} have been introduced.
Yet, none of the above methods and extensions natively support an incremental or dynamic setting in which the underlying model and its FI can rapidly change due to concept drift.

An initial approach to explaining model changes  by computing differences in FI utilizing drift detection methods is \cite{Muschalik.2022}.
However, this does not constitute an incremental FI measure.
The explanations are created with a time delay and without efficient anytime calculations.
A first step towards anytime FI values has been proposed for online Random Forests, where separate test data is used from online bagging to compute changes in impurity and accuracy \cite{Cassidy.2014}.
While this method is limited to online random forests, it does not provide theoretical guarantees or an incremental approach.
Similar to batch learning, incremental FI is also relevant to the field of incremental feature selection, where FI is calculated periodically with a sliding window to retain features for the incrementally fitted model \cite{Barddal.2019,Yuan.2018}.

In this work, we provide a truly incremental FI measure whose time sensitivity can be controlled by a single smoothing parameter. Moreover, we establish necessary theoretical guarantees on its approximation quality.

\section{Global Feature Importance}

We consider a supervised learning scenario, where $\mathscr X$ is the feature space and $\mathscr Y$ the target space, e.g., $\mathscr X = \mathbb{R}^d$ and $\mathscr Y = \mathbb{R}$ or $\mathscr Y = \{0,1\}$.
Let $h: \mathscr X \to \mathscr Y$ be a model, which is learned from a set or stream of observed data points $z=(x,y) \in \mathscr X \times \mathscr Y$.
Let $D = \{1,\dots,d\}$ be the set of feature indices for the vector-wise feature representations of $x = (x^{(i)}:i \in D) \in \mathscr X$. 
Consider a subset $S \subset D$ and its complement $\bar S := D \setminus S$, which partitions the features, and denote $x^{(S)} := (x^{(i)}: i \in S)$ as the feature subset of $S$ for a sample $x$. 
We write $h(x^{(\bar S)},x^{(S)}) := h(x)$ to distinguish between features from $\bar S$ and $S$. For the basic setting, we assume that $N$ observations are drawn independently and identically distributed (iid) from the joint distribution of unknown random variables $(X,Y)$ and denote by $\mathbb{P}_S$ the marginal distribution of the features in $S$, i.e.,
$z_n := (x_n,y_n)$ from $Z_n := (X_n,Y_n) \overset{iid}{\sim} \mathbb{P}_{(X,Y)}$ and $x_n^{(S)} \text{ from } X_n^{(S)} \overset{iid}{\sim} \mathbb{P}_S$
for samples $n=1,\dots,N$.

\emph{Feature importance}  refers to the relevance of a set of features $S$ for a model $h$.
To quantify FI, the key idea of measures such as PFI is to compare the model's performance when using only features in $\bar S$ with the performance when using all features in $D = S \cup \bar S$. 
The idea is that the ``removal'' of an important feature (i.e., the feature is not provided to a model) substantially decreases a model's performance.
The model performance or risk is measured based on a norm $\Vert \cdot\Vert :\mathscr Y \to \mathbb{R}$ on $\mathscr Y$, e.g., the Euclidean norm, as  $\mathbb{E}_{(X,Y)}[\, \|h(X)-Y\| \, ]$.

As the model is trained on all features and retraining is computationally expensive, a common method to restrict $h$ to $\bar S$ is to marginalize $h$ over the features in $S$.
We denote the marginalized risk
\begin{equation}\label{eq::marginalized-risk}
    f_S \big(x^{(\bar S)},y \big) := \mathbb{E}_{\tilde X \sim \mathbb{P}_S} \left[\Vert h(x^{(\bar S)},\tilde X)-y\Vert\right] .
\end{equation}
A popular way to define FI for a model $h$ and a feature set $S$ is to compare the marginalized risk with the model's inherent risk \cite{Covert_Lundberg_Lee_2020}. 
For a model $h$ and a subset $S \subset D$, FI becomes
\begin{equation*}
    \phi^{(S)}(h) := \underbrace{\mathbb{E}_{(X,Y)} \Big[f_S(X^{(\bar S)},Y) \Big]}_{\text{marginalized risk over $\mathbb{P}_S$}} - \underbrace{\mathbb{E}_{(X,Y)} \Big[\Vert h(X)-Y\Vert \Big]}_{\text{risk}}.
\end{equation*}
This FI measures the \emph{increase in risk} when the features in $S$ are marginalized 
\cite{Covert_Lundberg_Lee_2020}.

\subsubsection{Empirical estimation of FI.}
Given observations $(x_1,y_1),\dots,(x_N,y_N)$, we estimate the FI for a given model $h$ with the canonical estimator
\begin{equation}\label{eq::batch-fi-baseline}
    \hat\phi^{(S)}_{\varphi} := \frac 1 N\sum_{n=1}^N \hat\lambda^{(S)}(x_n,x_{\varphi(n)},y_n),
\end{equation}
where $\varphi:\{1,\dots,N\} \to \{1,\dots,N\}$ represents the realization of a (possibly random) sampling strategy that decides for each observation which observation should be taken to approximate $X^{(S)}$
and
\begin{equation*}
\hat\lambda^{(S)}(x_n,x_m,y_n) := \Vert h(x_n^{(\bar S)},x^{(S)}_m)-y_n\Vert- \Vert h(x_n)-y_n\Vert.   
\end{equation*}
Given the iid assumption, it is clear that due to $X_n \perp X_{n'}$ for $n\neq n'$, the estimator is an unbiased estimator of the FI $\phi^{(S)}(h)$, if $\varphi(n) \neq n$ for all $n=1,\dots,N$.
In the case of $\varphi(n)=n$, the term in the sum is zero as well as its expectation, which implies $\mathbb{E}[\hat\phi^{(S)}_{\varphi}] \leq \phi^{(S)}(h)$ for any $\varphi$.
We will now discuss a well understood choice of feature subsets $S \subset D$, sampling strategy $\varphi$ and two estimators for $\phi^{(S)}(h)$.

\subsection{Permutation Feature Importance (PFI)}\label{sec::batch-pfi}
A popular example of FI is the well-known PFI \cite{breiman_random_2001} that measures the importance of each feature $j \in D$ by using a set $S_j := \{j\}$.
More precisely, the FI for each feature $j \in D$ is given by $\phi^{(S_{j})}$ with sets  $S_{j} = \{j\}$  and their complement $\bar S_{j}= D \setminus \{j\}$.
The sampling strategy $\varphi$ used in PFI samples uniformly generated permutations $\varphi \in \mathfrak S_N$ over the set $\{1,\dots,N\}$, where each permutation has a probability of $1/N!$.

\subsubsection{(Empirical) PFI.}

Permutation tests, as proposed initially by \citet{breiman_random_2001}, effectively approximate $\mathbb{E}_{\varphi}[\hat\phi^{(S_j)}_\varphi]$ by averaging over $M$ uniformly sampled random permutations.
We introduce a scaled version of the initially proposed method as the PFI estimator, that is
\begin{equation}\label{eq::approx-sampling-pfi}
 \text{\textbf{PFI}: } \hat\phi^{(S_j)} := \frac N {N-1}\underbrace{\frac 1 M \sum_{m=1}^M \hat\phi^{(S_j)}_{\varphi_m}}_{\approx\mathbb{E}_{\varphi}[\hat\phi^{(S_j)}_\varphi]}
\end{equation}
with $\varphi_1,\dots,\varphi_m \overset{iid}{\sim} \text{unif}(\mathfrak S_N)$.
As discussed above, the estimator $\hat\phi^{(S_j)}_{\varphi}$ for a given $\varphi$ is an unbiased estimator for the FI $\phi^{(S_j)}(h)$, if the permutation is a derangement.  
In the following, we show that the above estimator is an unbiased estimator of FI, in contrast to the  original method without scaling.

\subsubsection{Expected PFI.}
The PFI estimator highly depends on the sampled permutations.
Therefore, we take the expectation over $\varphi$ to analyze its theoretical properties.
We can show that the expectation is the model reliance $\bar\phi^{(S_j)} := \hat e_{\text{switch}}-\hat e_{\text{orig}}$ which compares the model error $\hat e_{\text{orig}}=\frac 1 {N} \sum_{n=1}^N\Vert h(x_n)-y_n\Vert$ and 
the error of the model if averaged over all feature instantiations
$\hat e_{\text{switch}}=\frac 1 {N(N-1)} \sum_{n=1}^N\sum_{m \neq n}\Vert h(x_n^{(\bar{S_j})},x_m^{(S_j)})-y_n\Vert$. 
This quantity has been introduced and extensively studied by \citet{JMLR:v20:18-760}.
\footnote{As compared to \citet{JMLR:v20:18-760},  we consider the loss function $L(f,(y,x_n,x_m)) := \Vert h(x_n^{(\bar{S_j})},x_m^{(S_j)}) - y \Vert$ and denote $\bar\phi^{(S_j)} := \widehat{MR}_{\text{difference}}(h)$ in our case.}

\begin{theorem}\label{thm::batch-pfi}
The expected PFI (model reliance) can be rewritten as a normalized expectation over uniformly random permutations, i.e.
\begin{align}\label{eq::pfi-batch}
\bar\phi^{(S_j)} = \frac N {N-1}\mathbb{E}_{\varphi \sim \text{unif}(\mathfrak S_N)} \left[\hat\phi^{(S_j)}_{\varphi} \right].
\end{align}
\end{theorem}
Due to space restrictions, all proof are deferred to the supplementary material.
Both $\hat e_{\text{switch}}$ and $\hat e_{\text{orig}}$ as well as the estimator $\bar\phi^{(S_j)}$ are U-statistics, which implies unbiasedness, asymptotic normality and finite sample boundaries under weak conditions \cite{JMLR:v20:18-760}.
The variance can thus be directly computed and it is easy to show that $\mathbb{V}[\bar\phi^{(S_j)}]= \mathcal{O}(1/N)$, which by Chebyshev's inequality implies a bound on the approximation error as $\mathbb{P}(\vert \bar\phi^{(S_j)}- \phi^{(S_j)}(h)\vert > \epsilon) = \mathcal{O}(1/N)$.
Hence, the approximation error of the expected PFI is directly controlled by the number of observations $N$ used for computation.
The link between permutation tests and the U-statistic $\bar\phi^{(S_j)}$ was already discussed by \citet[Appendix A.3]{JMLR:v20:18-760}, where it was shown that the sum over permutations without fixed points is proportional to $\hat e_{\text{switch}}$.
The biased estimator $\frac 1 M \sum_{m=1}^M \hat\phi^{(S_j)}_{\varphi_m}$ appears in \cite{breiman_random_2001,JMLR:v20:18-760,Gregorutti_Michel_Saint-Pierre_2017}.
However, to our knowledge, the unbiased version in (\ref{eq::approx-sampling-pfi}) has not yet been introduced, and Theorem \ref{thm::batch-pfi} directly yields the unique normalizing constant $\frac N {N-1}$, which ensures that the estimator is unbiased.
In particular, Theorem \ref{thm::batch-pfi} justifies to average over repeatedly sampled realizations of $\varphi$ in order to approximate the computationally prohibitive estimator $\bar\phi^{(S_j)}$.
In the following, we will pick up this notion when constructing an incremental FI estimator.

\section{Incremental Permutation Feature Importance}
We now consider a sequence of models $(h_t)_{t\in \mathbb{N}}$ from an incremental learning algorithm.
At time $t$ the observed data is $\{(x_0,y_0),\dots,(x_t,y_t)\}$.
The model is incrementally learned over time, such that at time $t$ the observation $(x_t,y_t)$ is used to update $h_t$ to $h_{t+1}$.
Our goal is to efficiently provide an estimate of PFI at each time step $t$ for each feature $j \in D$ using subsets $S_j := \{j\}$.
Note that our results can immediately be extended to arbitrary feature subsets $S \subset D$.

In the following, we construct an efficient incremental estimator for PFI.
We first discuss how (\ref{eq::batch-fi-baseline}) can be efficiently approximated in the incremental learning scenario, given a sampling strategy $\varphi_t$. 
In the sequel, we will rely on a random sampling strategy which is specifically suitable for the incremental setting and easier to implement than permutation-based approaches. 
Note that a permutation-based approach at time $t$ is difficult to replicate in the incremental setting, as at time $s<t$ not all samples until time $t$ are available.
As the model changes over time, naively computing (\ref{eq::batch-fi-baseline}) at each time step $t$ using $N$ previous observations results in $N$ model evaluations per time step.
Instead, we aim for an estimator that averages the terms in (\ref{eq::batch-fi-baseline}) over time rather than over multiple data points at one time step, i.e., we evaluate the current model only twice to compute the time-dependent quantity
\begin{equation*}
    \hat\lambda^{(S_j)}_t(x_t,x_{\varphi_t},y_t) := \Vert h_t(x_t^{(\bar{S_j})},x^{(S_j)}_{\varphi_t})-y_t\Vert- \Vert h_t(x_t)-y_t\Vert,
\end{equation*}
where $\varphi_t: \Omega \to \{0,\dots,t-1\}$ is a sampling strategy to select a previous observation.
We propose to average these calculations over time by using exponential smoothing, i.e.
\begin{equation*}
        \text{\textbf{iPFI}: } \hat\phi^{(S_j)}_t := (1-\alpha)\hat\phi^{(S_j)}_{t-1} + \alpha \hat\lambda^{(S_j)}_t(x_t,x_{\varphi_t},y_t),
\end{equation*}
for $t>t_0$, $\hat\phi^{(S_j)}_{t_0} := \hat\lambda^{(S_j)}_{t_0}(x_t,x_{\varphi_{t_0}},y_{t_0})$, and $\alpha\in(0,1)$.
The parameter $\alpha$ is a hyperparameter that should be chosen based on the application.
Note that a specific choice of $\alpha$ corresponds to a window size $N$, where $\alpha = \frac 2 {N+1}$ based on the well-known conversion formula, see e.g.\  \cite[p.73]{nahmias2015production}.
Given a realization $\varphi_s$, observations $z_s := (x_s,y_s)$ from iid $Z_s := (X_s,Y_s) \overset{iid}{\sim} \mathbb{P}_{(X,Y)}$ and $x_s^{(S_j)}$ from $X_s^{(S_j)} \overset{iid}{\sim} \mathbb{P}_{S_j}$, each $\hat\lambda_s^{(S_j)}$ is an unbiased estimate of $\phi^{(S_j)}(h_s)$.
We further require $\varphi_s \perp (X,Y)$ and denote
\begin{equation}\label{eq::inc-fi-sampling-procedure}
    \varphi_s: \Omega \to \{0,\dots,s-1\} \text{ with } p_{s,r} := \mathbb{P}(\varphi_s = r),
\end{equation}
for $s=t_0,\dots,t$ to select previous observations.
Note that $t_0>0$ is the first time step where $\hat\phi^{(S_j)}_t$ can be computed, as we need previous observations for the sampling process.
In the following, we assume that the sampling strategy $(\varphi_s)_{t_0\leq s\leq t}$ is fixed and clear from the context, and thus omit the dependence in $\hat\phi^{(S_j)}_t$.
We illustrate one explanation step at time $t$ in Algorithm \ref{alg:incremental_fi}.
This directly corresponds to (\ref{eq::approx-sampling-pfi}) with $M=1$ and can be extended to $M>1$ by repeatedly running the procedure in parallel and averaging the results.
Next, we discuss two possible sampling strategies.

\begin{algorithm}[t]
\caption{iPFI explanation at time $t$ for feature $j$}
\label{alg:incremental_fi}
\begin{algorithmic}[1]
\Require: $\alpha \in (0,1)$, sampling strategy $\varphi_t$, and $\hat\phi^{(S_j)}_{t-1}$.
\Procedure{explainOne}{$h_t,x_t,y_t,j$}
    \State $x_{s} \gets \text{Sample}(\varphi_t)$ 
    \State $\hat\lambda^{(S_j)}_{t} \gets \Vert h_{t}(x_t^{(\bar{S_j})},x_{s}^{(S_j)})-y_t\Vert -\Vert h_{t}(x_t)-y_t\Vert$ 
    \State$ \hat\phi^{(S_j)}_t \gets (1-\alpha) \cdot \hat\phi^{(S_j)}_{t-1} + \alpha \cdot \hat\lambda^{(S_j)}_{t}$ 
    \State $\varphi_{t+1} \gets \text{UpdateSampler}(\varphi_{t},x_t)$ 
\EndProcedure
\end{algorithmic}
\end{algorithm}

\subsection{Incremental Sampling Strategies $\varphi$}\label{sec::sampling-procedures}
Since random permutations cannot easily be realized in an incremental setting as they would require knowledge of future events, we now present two alternative types of sampling strategies.
We formalize $(\varphi_s)_{t_0\leq s\leq t}$ to choose the previous observation $r$ at time $s$ for the calculation in $\hat\lambda_s^{(S_j)}$.
To do so, we will specify the probabilities $p_{s,r}$ in (\ref{eq::inc-fi-sampling-procedure}).

\subsubsection{Uniform Sampling}
In uniform sampling we assume that each previous observation is equally likely to be sampled at time $s$, i.e., $p_{s,r}=1/s$ for $s=t_0,\dots,t$ and $r=0,\dots,s-1$.
It can be naively implemented by storing all previous observations and uniformly sampling at each time step.
However, when memory is limited, it can be implemented with histograms for features of known cardinality.
For others, a reservoir of fixed length can be maintained, known as reservoir sampling \cite{10.1145/3147.3165_vitter}.
The probability of a new observation to be included in the reservoir then decreases over time.
Clearly, observations are drawn independently, but can be sampled more than once.
In a data stream scenario, where changes to the underlying data distribution occur over time, the uniform sampling strategy may be inappropriate, and sampling strategies that prefer recent observations may be better suited.

\subsubsection{Geometric Sampling}
Geometric sampling arises from the idea to maintain a reservoir of size $L$, which is updated by a new observation at each time step by randomly replacing a reservoir observation with the newly observed one.
Until time $t_0$ the first $L$ observations are stored in the reservoir.
At each sampling step ($t \geq t_0$) an observation is uniformly chosen from the reservoir with probability $p := 1/L$.
Independently, a sample from the reservoir is selected with the same probability $p := 1/L$ for replacement with the new observation. 
The resulting probabilities are of the geometric form~$p_{s,r}=p(1-p)^{s-r-1}$ for $r\geq t_0$ and $p_{s,r}=p(1-p)^{s-t_0}$ for~$r < t_0$.
Clearly, the geometric sampling strategy yields increasing probabilities for more recent observations and we demonstrate in our experiments that this can be beneficial in scenarios with concept drift.

\subsection{Theoretical Results of Estimation Quality}\label{sec::constant-h}
The estimator $\hat\phi^{(S_j)}_t$ picks up the notion of the PFI estimator $\hat\phi^{(S_j)}$ in (\ref{eq::approx-sampling-pfi}), which approximates the expectation over the random sampling strategy $(\varphi)_{t_0 \leq s \leq t}$ by averaging repeated realizations.
While $\hat\phi^{(S_j)}_t$ only considers one realization of the sampling strategy, it is easy to extend the approach in the incremental learning scenario by computing the estimator $\hat\phi^{(S_j)}_t$ in multiple separate runs in parallel.
While this yields an efficient estimate of PFI, it is difficult to analyze the estimator theoretically as each estimator highly depends on the realizations of the sampling strategy.
We thus again study the expectation over the sampling strategy and introduce the \emph{expected} iPFI as $\bar\phi^{(S_j)}_t := \mathbb{E}_\varphi[\hat\phi^{(S_j)}_t]$, similar to the expected PFI (model reliance) $\bar\phi^{(S_j)}$.
To evaluate the estimation quality, we will analyze the bias $\vert \bar\phi^{(S_j)}_t - \phi^{(S_j)}(h_t) \vert$ and the variance of $\bar\phi^{(S_j)}_t$.
Both can be combined by Chebyshev's inequality to obtain bounds on the approximation error of $\phi^{(S_j)}(h_t)$ for $\epsilon > \vert \bar\phi^{(S_j)}_t - \phi^{(S_j)}(h_t) \vert$ as
\begin{equation}\label{eq::approximation-errror}
\mathbb{P}(\vert \bar\phi^{(S_j)}_t - \phi^{(S_j)}(h_t) \vert > \epsilon)= \mathcal O (\mathbb{V}[\bar\phi^{(S_j)}_t]).
\end{equation}
As already said, all proofs are deferred to the supplementary material.
Our theoretical results are stated and proven in a  general manner, which allows one to extend our approach to other sampling strategies, other feature subsets, and even other aggregation techniques.

\paragraph{Static Model.}
Given iid observations from a data stream, we consider an incremental model that learns over time.
We begin under the simplified assumption that the model does not change over time, i.e., $h_t \equiv h$ for all $t$.

\begin{theorem}[Bias for static Model]\label{thm::constant-h-bias}
If $h \equiv h_t$, then 
\begin{equation*}
 \phi^{(S_j)}(h) - \bar\phi^{(S_j)}_t = (1-\alpha)^{t-t_0+1} \phi^{(S_j)}(h).
 \end{equation*}
\end{theorem}

From the above theorem it is clear that the bias of the expected iPFI $\bar\phi^{(S_j)}_t$ is exponentially decreasing towards zero for $t \to \infty$ and we thus continue to study the asymptotic estimator $\lim_{t\to  \infty}\bar\phi_t^{(S_j)}$.
While the bias does not depend on the sampling strategy, our next results analyzes the variance of the asymptotic estimator, which does depend on the sampling strategy.

\begin{theorem}[Variance for static Model]\label{thm::constant-h-variance}
If $h_t \equiv h$ and $\mathbb{V}[\Vert h(X_s^{(\bar{S_j})},X_r^{(S_j)})-Y_s\Vert -\Vert h(X_s)-Y_s\Vert] <\infty$, then
\begin{align*}
\text{Uniform: } &\mathbb{V} \left[ \lim_{t\to \infty}\bar\phi_t^{(S_j)} \right] = \mathcal O (-\alpha\log(\alpha)).
\\
\text{Geometric: } &\mathbb{V} \left[ \lim_{t\to \infty}\bar\phi_t^{(S_j)} \right] = \mathcal O (\alpha) + \mathcal O (p).
\end{align*}
\end{theorem}
The variance is therefore directly controlled by the choice of parameters $\alpha$ and $p$.
As the asymptotic estimator is unbiased, it is clear that these parameters control the approximation error, as shown in (\ref{eq::approximation-errror}).

\paragraph{Changing Model.}
So far, we discussed properties of $\bar\phi_t^{(S_j)}$ under the simplified assumption that $h_t$ does not change over time.
In an incremental learning scenario, $h_t$ is updated incrementally at each time step.
In cases where no concept drift affects the underlying data generating distribution, we can assume that an incremental learning algorithm gradually converges to an optimal model.
We thus assume that the change of the model is controlled and show  results similar to the case where $h_t$ is static.
To control model change formally, we introduce $f^{\Delta}_S(x^{(\bar{S_j})},h_s,h_t) := \mathbb{E}_{\tilde X \sim \mathbb{P}_S}[\Vert h_t(x^{(\bar{S_j})},\tilde X)-h_s(x^{(\bar{S_j})},\tilde X)\Vert]$.
The expectation of $f^\Delta_S$ is denoted $\Delta_S(h_s,h_t) := \mathbb{E}_X[f^{\Delta}_S(X,h_s,h_t)]$
and $\Delta(h_s,h_t) := \Delta_\emptyset(h_s,h_t)$.
We show that $\Delta_S$ and $\Delta$ bound the difference of FI of two models $h_t$ and $h_s$ and the bias of our estimator.
\begin{theorem}[Bias for changing Model]\label{thm::changing-h-bias}
If $\Delta(h_s,h_t) \leq \delta$ and $\Delta_S(h_s,h_t) \leq \delta_S$ for $t_0 \leq s \leq t$, then
\begin{equation*}
 \vert \bar\phi^{(S_j)}_t - \phi^{(S_j)}(h_t)\vert \leq \delta_S + \delta +\mathcal O((1-\alpha)^{t}).   
\end{equation*}
\end{theorem}

In the case of a changing model the estimator is therefore only unbiased if $h_t \to h$ as $t \to \infty$.
For results on the variance, we control the variability of the models at different points in time.
In the case of a static model, the covariances can be uniformly bounded, as they do not change over time.
Instead, for a changing model, we introduce the time-dependent function 
$$
f_s(Z_s,Z_r) := \Vert h_s(X_s^{(\bar{S_j})},X_r^{(S_j)})-Y_s\Vert -\Vert h_s(X_s)-Y_s\Vert
$$ 
and assume existence of some $\sigma_{\text{max}}^2$ such that
\begin{equation}\label{eq::bound-variability}
    \text{cov}(f_s(Z_s,Z_r),f_{s'}(Z_{s'},Z_{r'})) 
    \leq \sigma_{\text{max}}^2
\end{equation}
for $t_0\leq s,s' \leq t$, $r<s$ and $r'<s'$.

\begin{theorem}[Variance for changing Model]\label{thm::changing-h-variance}
Given (\ref{eq::bound-variability}) for a sequence of models $(h_t)_{t\geq0}$, the results of Theorem \ref{thm::constant-h-variance} apply.
\end{theorem}

\paragraph{Summary.}
We have shown that the approximation error of iPFI for FI is controlled by the parameters $\alpha$ and $p$.
In the case of drifting data, the approximation error is additionally affected by the changes in the model, as it is then possibly biased and the covariances may change.
As the expected PFI estimator has an approximation error of order $\mathcal O(1/N)$ for FI, we conclude that the above bounds on the approximation error of expected iPFI are also valid when compared with the expected PFI, if $\alpha$ is chosen according to $\alpha = \frac{2}{N+1}$.
In the next section, we corroborate our theoretical findings with empirical evaluations and showcase the efficacy of iPFI in scenarios with concept drift. We also elaborate on the differences between the two sampling strategies.

\newpage
\section{Experiments}
We conduct multiple experimental studies to validate our theoretical findings and present our approach on real data.
We consider three benchmark datasets, which are well-established in the FI literature \cite{Covert_Lundberg_Lee_2020,NIPS2017_7062}, one real-world data stream, and one synthetic data stream.
As our approach is inherently model-agnostic, we present experimental results for different model types.
As classification problems, we use \emph{adult} \cite{adult_1996} with a Gradient Boosting Tree (GBT) \cite{Friedman_2001} and \emph{bank} \cite{moro2011using} with a small 2-layer Neural Network (NN) with layer sizes $(128,64)$. 
As a regression problem, we use \emph{bike} \cite{fanaee2014event} with LightGBM (LGBM) \cite{ke2017lightgbm}.
The real-world electricity-price classification data stream mentioned in the introduction is called \emph{elec2} \cite{Harries99splice-2comparative}. In the static case an LGBM model performed best and in the online setting an Adaptive Random Forest classifier (ARF) \cite{gomes2017adaptive} was used.
The synthetic data stream is constructed with the \emph{agrawal} \cite{Agrawal.1993} classification data generator.
Like \emph{elec2}, an LGBM was used in the static scenario, and an ARF was applied in the dynamic setting.
The models' and data streams' implementation is based on \textit{scikit-learn} \cite{scikit-learn}, \textit{River} \cite{montiel_river_2020}, and \textit{OpenML} \cite{OpenML2020}.
We mainly rely on default parameters, and the supplement contains detailed information about the datasets and applied models.
\\
In all our experiments, we compute the \textbf{iPFI} estimator $\hat{\phi}_{\text{iPFI}}^{(S_j)}$ as the average over ten realizations $\hat\phi^{(S_j)}_t$ of the incremental sampling strategies (uniform or geometric).
All baseline approaches are chosen, such that they require the same amount of model evaluations as iPFI.

\begin{figure*}[t]
\begin{minipage}{\textwidth}
    \begin{minipage}[b]{0.59\textwidth}
        \centering
        \includegraphics[width=\textwidth]{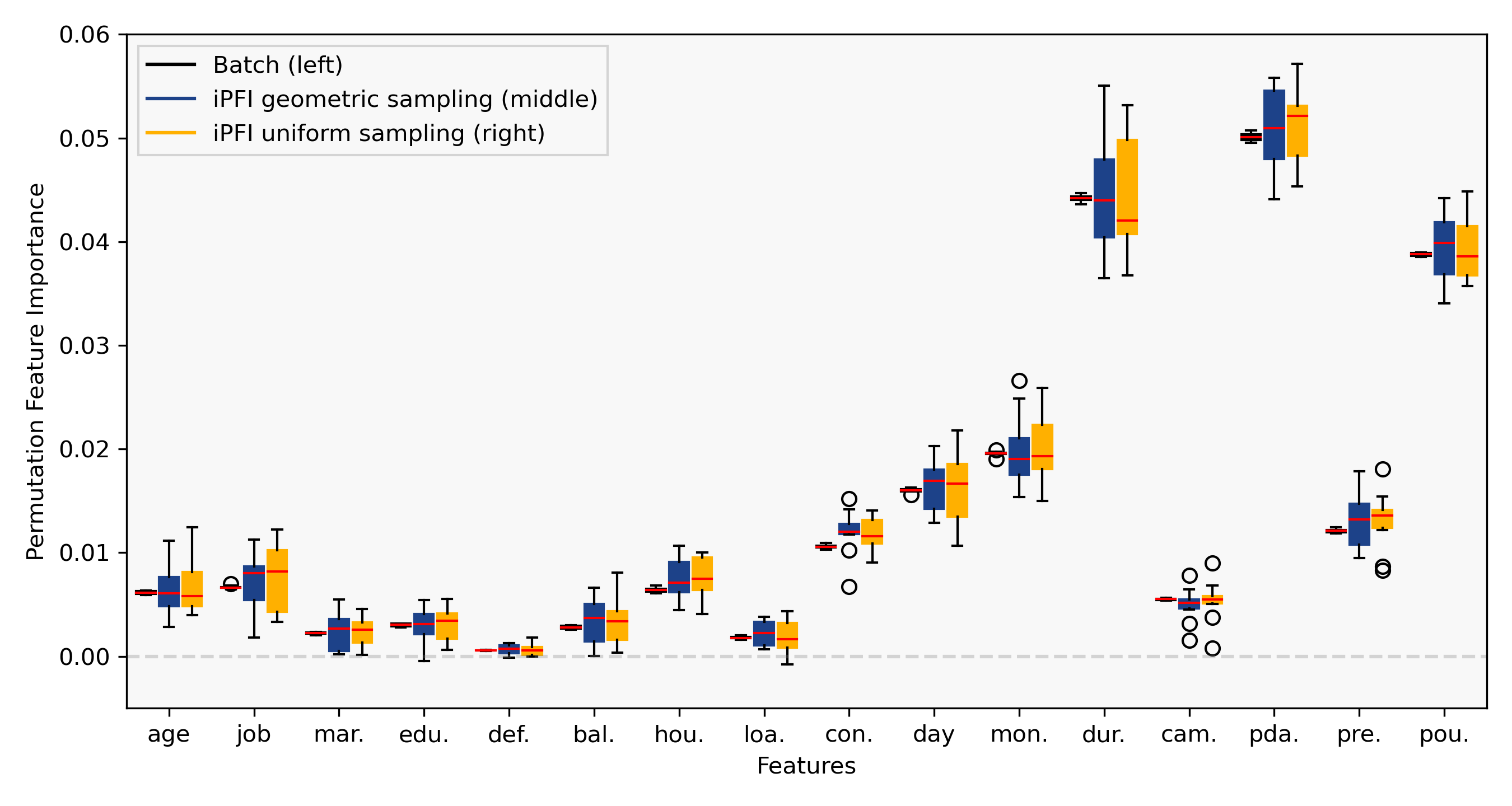}
        \captionof{figure}{Boxplot of PFI estimates per feature of the \emph{bank} dataset for batch PFI (left), geometric sampling iPFI (middle), and uniform sampling iPFI (right) on a pre-trained static NN.}
        \label{fig:static_model}
    \end{minipage}
    \hfill
    \begin{minipage}[b]{0.39\textwidth}
        \centering
        \begin{tabular}{@{}cccc@{}}
        \toprule
        \multirow{2}{*}{\begin{tabular}[c]{@{}c@{}}\textbf{data}\\ (N)\end{tabular}} & \multirow{2}{*}{\begin{tabular}[c]{@{}c@{}}\textbf{model}\\ (perf.)\end{tabular}} & \multicolumn{2}{c}{\textbf{error}} \\
         &  & uniform & geometric \\ \midrule
        \begin{tabular}[c]{@{}c@{}}agrawal\\ (20k)\end{tabular} & \begin{tabular}[c]{@{}c@{}}LGBM\\ (99\%)\end{tabular} & \begin{tabular}[c]{@{}c@{}}0.011\\ (.006)\end{tabular} & \begin{tabular}[c]{@{}c@{}}0.010\\ (.006)\end{tabular} \\[9pt]
        \begin{tabular}[c]{@{}c@{}}elec2\\ ($\approx$45k)\end{tabular} & \begin{tabular}[c]{@{}c@{}}LGBM\\ (88\%)\end{tabular} & \begin{tabular}[c]{@{}c@{}}0.038\\ (.012)\end{tabular} & \begin{tabular}[c]{@{}c@{}}0.037\\ (.011)\end{tabular} \\[9pt]
        \begin{tabular}[c]{@{}c@{}}adult\\ ($\approx$45k)\end{tabular} & \begin{tabular}[c]{@{}c@{}}GBT\\ (86\%)\end{tabular} & \begin{tabular}[c]{@{}c@{}}0.126\\ (.040)\end{tabular} & \begin{tabular}[c]{@{}c@{}}0.114\\ (.025)\end{tabular} \\[9pt]
        \begin{tabular}[c]{@{}c@{}}bank\\ ($\approx$45k)\end{tabular} & \begin{tabular}[c]{@{}c@{}}NN\\ (91\%)\end{tabular} & \begin{tabular}[c]{@{}c@{}}0.126\\ (.024)\end{tabular} & \begin{tabular}[c]{@{}c@{}}0.132\\ (.013)\end{tabular} \\[9pt]
        \begin{tabular}[c]{@{}c@{}}bike\\ ($\approx$17k)\end{tabular} & \begin{tabular}[c]{@{}c@{}}LGBM\\ (26.6)\end{tabular} & \begin{tabular}[c]{@{}c@{}}0.022\\ (.005)\end{tabular} & \begin{tabular}[c]{@{}c@{}}0.019\\ (.008)\end{tabular} \\ \bottomrule
        \end{tabular}
        \captionof{table}{Median error of iPFI compared to batch PFI (IQR between $Q_{1}$ and $Q_{3}$ in braces). Model performance is measured in accuracy and mean absolute error (\emph{bike}).}
        \label{tab:exp_a_mae}
    \end{minipage}
\end{minipage}
\end{figure*}

\begin{figure*}[t]
    \centering
    \includegraphics[width=\textwidth]{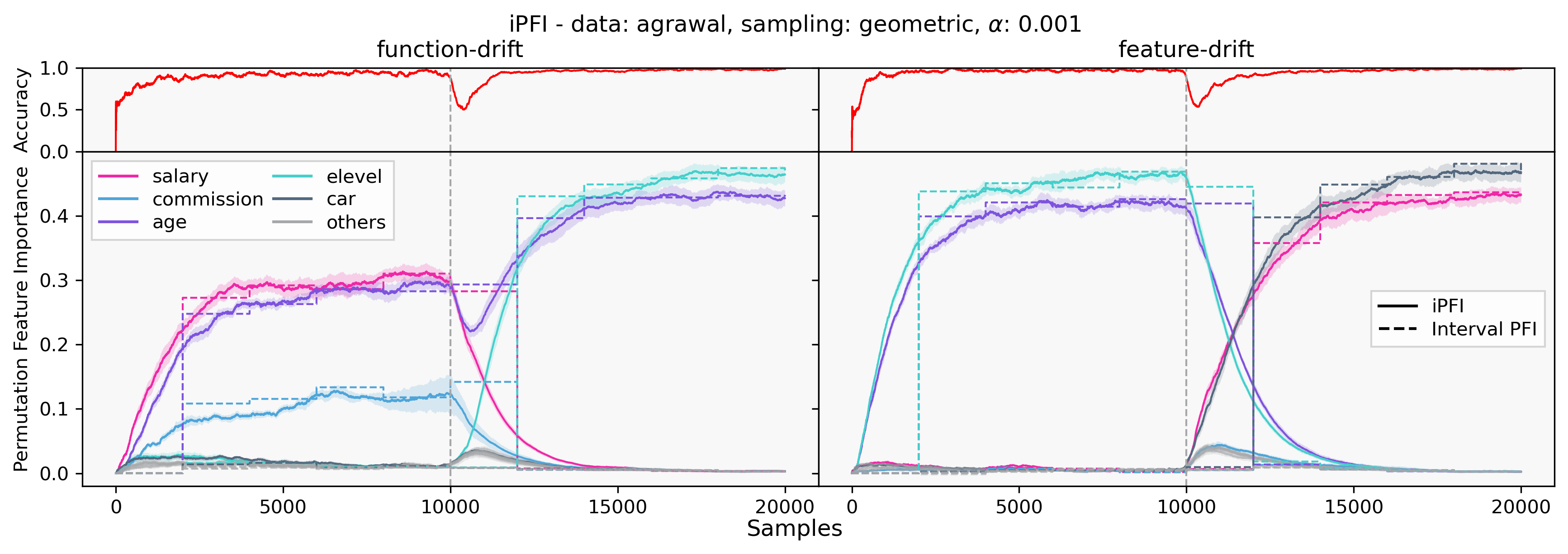}
    \caption{iPFI on two \textit{agrawal} data streams with induced concept drift. The most important features are colored. The dashed line denotes the batch calculation at set intervals. The dashed vertical line denotes the concept drift.}
    \label{fig:exp_b}
\end{figure*}

\begin{figure*}[t]
    \centering
    \includegraphics[width=\textwidth]{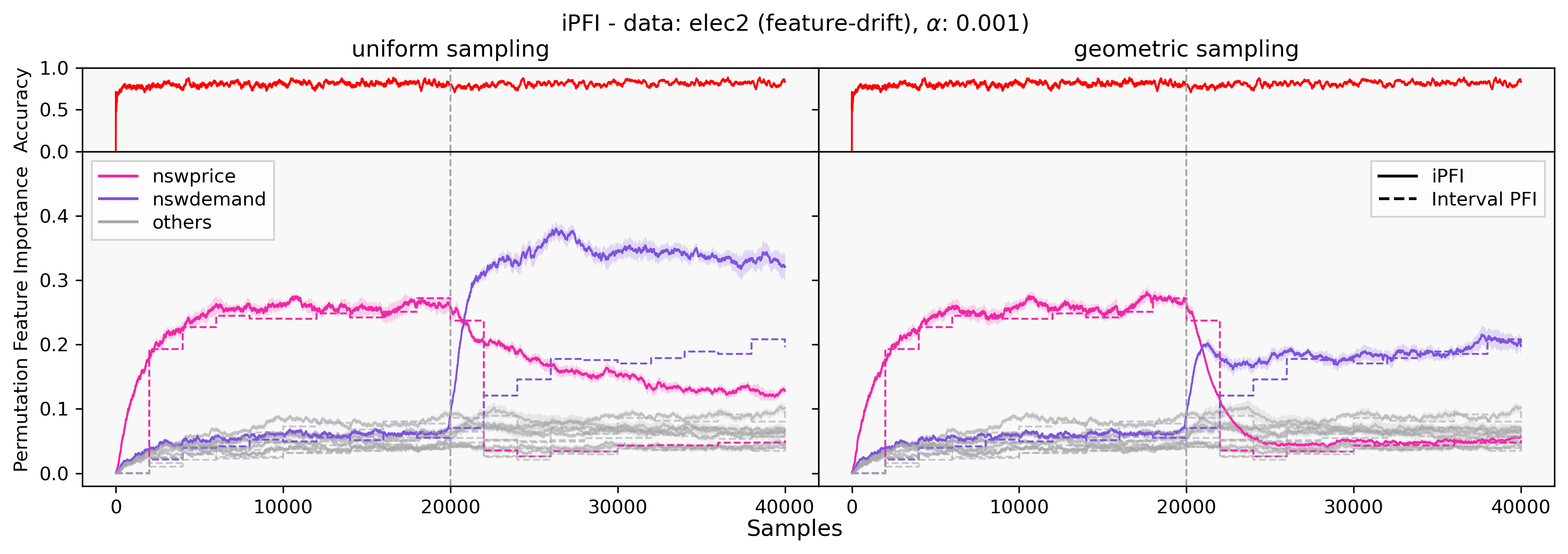}
    \caption{iPFI with uniform (left) and geometric sampling (right) on \textit{elec2} with a feature-drift.}
    \label{fig:uniform_vs_geometric}
\end{figure*}

\subsection{Experiment A: Approximation of Batch PFI}
First, we consider the static model setting where models are pre-trained before they are explained on the whole dataset (no incremental learning).
This experiment demonstrates that iPFI correctly approximates batch PFI estimation.
\\
We compare iPFI with the classical \textbf{batch PFI} $\hat\phi^{(S_j)}_{\text{batch}}$ for feature $j \in D$, which is computed using the whole static dataset over ten random permutations.
We normalize $\hat{\phi}_{\text{iPFI}}^{(S_j)}$ and $\hat{\phi}_{\text{batch}}^{(S_j)}$ between 0 and 1, and compute the sum over the feature-wise absolute approximation errors;
\begin{equation*}
    \text{error} := \sum_{j \in D}{\left|\hat{\phi}_{\text{iPFI}}^{(S_j)}-\hat{\phi}_{\text{batch}}^{(S_j)}\right|}.
\end{equation*} 
Table~\ref{tab:exp_a_mae} shows the median and interquartile range (IQR) (difference between the first and third quartile) of the error based on ten random orderings of each dataset.
Figure~\ref{fig:static_model} shows the approximation quality of iPFI with geometric and uniform sampling per feature for the \emph{bank} dataset. 
In the static modeling case, there is no clear difference between geometric and uniform sampling.
However, in the dynamic modeling context under drift, the sampling strategy has a substantial effect on the iPFI estimates. 

\subsection{Experiment B: Online PFI Calculation under Drift}
In this experiment, we consider a dynamic modeling scenario. 
Here, instead of a pre-trained model, we fit ARF models incrementally on real data streams and compute iPFI on the fly.
For the sake of clarity and simplicity, we only present results for ARF models here. 
However, as our approach is inherently model-agnostic, any incremental model (implemented for example in \emph{river}) can be explained.
As a baseline, we compare our approach to the \textbf{interval PFI} for feature $j \in D$, which computes the PFI over fixed time intervals during the online learning process with ten random permutations in each interval.
This can be seen as a naive implementation of iPFI with large gaps of uncertainty and a substantial time delay.
With the synthetic \emph{agrawal} stream we induce two kinds of \emph{real} concept drifts:
First, we switch the classification function of the data generator, which we refer to as function-drift (changing the functional dependency but retaining the distribution of $X$).
Second, we switch the values of two or more features with each other, which we refer to as feature-drift (changing the functional dependency by changing the distribution of $X$).
Note that feature-drift can also be applied to datasets, where the classification function is unknown.
Figure~\ref{fig:exp_b} showcases how well iPFI reacts to both concept drift scenarios.
Both concept drifts are induced in the middle of the data stream (after 10,000 samples).
For the function-drift example (Figure~\ref{fig:exp_b}, left), the \emph{agrawal} classification function was switched from \citet{Agrawal.1993}'s concept $1$ to concept $2$.
Theoretically, only two features should be important for both concepts: 
For the first concept the pink \emph{salary} and the purple \emph{age} features are needed, and for the second concept the classification function relies on the cyan \emph{education} and the purple \emph{age} features.
However, the ARF model also relies on the blue \emph{commission} feature, which can be explained as \emph{commission} directly depends on \emph{salary}.
In the feature-drift scenario (Figure~\ref{fig:exp_b}, right), the ARF model adapts to a sudden drift where both important features (\emph{education} and \emph{age}) are switched with two unimportant features (\emph{car} and \emph{salary}). 
In both scenarios iPFI instantly detects the shifts in importance.  

From both simulations, it is clear that iPFI and its anytime computation has clear advantages over interval PFI.
In fact, iPFI quickly reacts to changes in the data distribution while still closely matching the ``ground-truth'' results of the batch-interval computation.
For further concept drift scenarios, we refer to the supplementary material.

\subsection{Experiment C: Geometric vs. Uniform Sampling}
Lastly, we focus on the question, which sampling strategy to prefer in which learning environments.
We conclude that geometric sampling should be applied under feature-drift scenarios, as the choice of sampling strategy substantially impacts iPFI's performance in concept drift scenarios where feature distributions change.
If a dynamic model adapts to changing feature distributions, and the PFI is estimated with samples from the outdated distribution, the resulting replacement samples are outside the current data manifold. 
Estimating PFI by using this data can result in skewed estimates, as illustrated in Figure~\ref{fig:uniform_vs_geometric}.
There, we induce a feature-drift by switching the values of the most important feature for an ARF model on \emph{elec2} with a random feature. 
The uniform sampling strategy (Figure~\ref{fig:uniform_vs_geometric}, left) is incapable of matching the ``ground-truth'' interval PFI estimation like the geometric sampling strategy (Figure~\ref{fig:uniform_vs_geometric}, right).
Hence, in dynamic learning environments like data stream analytics or continual learning, we recommend applying a sampling strategy that focuses on more recent samples, such as geometric distributions.
For applications without drift in the feature-space like progressive data science, uniform sampling strategies, which evenly distribute the probability of a data point being sampled across the data stream, may still be preferred.

For further experiments on different parameters, we again refer to the supplement.
Therein, we show that the smoothing parameter $\alpha$ substantially effects iPFI's FI estimates.
Like any smoothing mechanism, this parameter controls the deviation of iPFI's estimates.
This parameter should be set individually for the task at hand.
In our experiment, values between $\alpha = 0.001$ (conservative) and $\alpha = 0.01$ (reactive) appeared to be reasonable.

\section{Conclusion and Future Work}
In this work, we considered global FI as a statistic measure of change in the model's risk when features are marginalized.
We discussed PFI as an approach to estimate feature importance and proved that only appropriately scaled permutation tests are unbiased estimators.
In this case, the expectation over the sampling strategy (\emph{expected PFI}) then corresponds to the model reliance U-Statistic \cite{JMLR:v20:18-760}.

Based on this notion, we presented iPFI, an efficient model-agnostic algorithm to incrementally estimate FI by averaging over repeated realizations of a sampling strategy.
We introduced two incremental sampling strategies and established theoretical results for the expectation over the sampling strategy (\emph{expected iPFI}) to control the approximation error using iPFI's parameters.
On various benchmark datasets, we demonstrated the efficacy of our algorithms by comparing them with the batch PFI baseline method in a static progressive setting as well as with interval-based PFI in a dynamic incremental learning scenario with different types of concept drift and parameter choices.

Applying XAI methods incrementally to data stream analytics offers unique insights into models that change over time.
In this work, we rely on PFI as an established and inexpensive FI measure.
Other computationally more expensive approaches (such as SHAP) address some limitations of PFI.
As our theoretical results can be applied to arbitrary feature subsets, analyzing these methods in the dynamic environment offers interesting research opportunities. 
In contrast to this work's technical focus, analyzing the dynamic XAI scenario through a human-focused lens with human-grounded experiments is paramount \cite{DoshiVelez.2017}.

\newpage

{
\small
\bibliography{main.bib} 
}

\section{Acknowledgments} We gratefully acknowledge funding by the Deutsche Forschungsgemeinschaft (DFG, German Research Foundation): TRR 318/1 2021 – 438445824.

\newpage
~\newpage

\appendix
\newtheorem{proposition}{Proposition}
\newtheorem{lemma}{Lemma}

\begin{center}
    \textbf{Technical Supplement \\ Incremental Permutation Feature Importance (iPFI): \\ Towards Online Explanations on Data Streams}
\end{center}

This is the technical supplement for the contribution \emph{Incremental Permutation Feature Importance (iPFI): Towards Online Explanations on Data Streams}. The supplement contains proofs to all of our theoretical claims, a description about the datasets and models used, summary information to the contribution's experiments, and a variety of further experiments.

\section{Proofs}
In the following, we provide the proofs of all theorems. We further present more general results that are stated as propositions.

\begin{theorem}\label{thm::batch-pfi}
The expected PFI (model reliance) can be rewritten as a normalized expectation over uniformly random permutations, i.e.
\begin{align}\label{eq::pfi-batch}
\bar\phi^{(S_j)} = \frac N {N-1}\mathbb{E}_{\varphi \sim \text{unif}(\mathfrak S_N)} \left[\hat\phi^{(S_j)}_{\varphi} \right].
\end{align}
\end{theorem}

\begin{proof}
We write $f(z_n,z_m) := \Vert h(x_n^{(\bar{S_j})},x_m^{(S_j)})-y_n\Vert -\Vert h(x_n)-y_n\Vert$ and compute the expectation over randomly sampled permutations $\varphi \in \mathfrak{S}_N$. 
Each permutation has probability $\frac{1}{N!}$, which yields
\begin{align*}
    \mathbb{E}_{\varphi}[\hat\phi^{(S_j)}_\varphi] &= \frac 1 {N!} \sum_{\varphi \in \mathfrak{S}_N}\hat\phi^{(S_j)}_\varphi
    \\
    &= \frac 1 N \frac 1 {N!} \sum_{n=1}^N \sum_{\varphi \in \mathfrak{S}_N}  f(z_n,z_{\varphi(n)}) 
    \\
    &= \frac 1 N \frac 1 {N!} \sum_{n=1}^N \sum_{m=1}^N (N-1)!  f(z_n,z_m) 
    \\
    &= \frac 1 N \frac 1 {N} \sum_{n=1}^N \sum_{m \neq n}  f(z_n,z_m)
    \\
    &= \frac 1 N \frac 1 {N} \sum_{n=1}^N \sum_{m \neq n} \Vert h(x_n,x_{m})-y_n\Vert 
    \\
    &- \frac {N-1} {N^2} \sum_{n=1}^N \Vert h(x_n)-y_n\Vert,
\end{align*}
where we used in the third line that there are $(N-1)!$ permutations with $\varphi(n)=m$.
We thus conclude,
\begin{align*}
\frac N {N-1} \mathbb{E}_{\varphi}[\hat\phi^{(S_j)}_\varphi] &= \hat e_{\text{switch}} -\hat e_{\text{orig}} = \bar\phi^{(S_j)}.
\end{align*}

\end{proof}

\begin{theorem}[Bias for static Model]\label{thm::constant-h-bias}
If $h \equiv h_t$, then 
\begin{equation*}
 \phi^{(S_j)}(h) - \bar\phi^{(S_j)}_t = (1-\alpha)^{t-t_0+1} \phi^{(S_j)}(h).
 \end{equation*}
\end{theorem}

\begin{proof}
We consider the more general estimator $\tilde \phi^{(S)}_t := \mathbb{E}_{\varphi}[\sum_{s=t_0}^t w_s \hat\lambda^{(S)}_t(x_t,x_{\varphi_t},y_t)]$ and prove a more general result that can be used for arbitrary sampling and aggregation techniques.
\begin{proposition}
If $h \equiv h_t$, then 
\begin{equation*}
 \phi^{(S)}(h) - \mathbb{E}[\tilde\phi^{(S)}_t] = (1-\mu_w) \phi^{(S)}(h)    
\end{equation*}
with $\mu_w := \sum_{s=t_0}^t w_s$.
\end{proposition}

\begin{proof}
As each $\hat\lambda_s^{(S)}$ is an unbiased estimator of $\phi^{(S)}(h_s)$, we have $\mathbb{E}[\tilde\phi^{(S)}_t]=\sum_{s=t_0}^{t} w_s \phi^{(S)}(h)= \mu_w \phi^{(S)}(h)$, where we used $(\varphi)_{t_0 \leq s \leq t} \perp (X,Y)$.
\end{proof}
The result then follows directly, as $\bar\phi^{(S)} = \tilde\phi^{(S)}$ for $w_s := \alpha(1-\alpha)^{t-s}$, $\mu_w=1-(1-\alpha)^{t-t_0+1}$ and $S := S_j$.

\end{proof}

\begin{theorem}[Variance for static Model]\label{thm::constant-h-variance}
If $h_t \equiv h$ and $\mathbb{V}[\Vert h(X_s^{(\bar{S_j})},X_r^{(S_j)})-Y_s\Vert -\Vert h(X_s)-Y_s\Vert] <\infty$, then
\begin{align*}
\text{Uniform: } &\mathbb{V} \left[ \lim_{t\to \infty}\bar\phi_t^{(S_j)} \right] = \mathcal O (-\alpha\log(\alpha)).
\\
\text{Geometric: } &\mathbb{V} \left[ \lim_{t\to \infty}\bar\phi_t^{(S_j)} \right] = \mathcal O (\alpha) + \mathcal O (p).
\end{align*}
\end{theorem}

\begin{proof}
We again consider the more general estimator $\tilde \phi^{(S)}_t  := \mathbb{E}_\varphi[\sum_{s=t_0}^t w_s \hat\lambda^{(S)}_t(x_t,x_{\varphi_t},y_t)]$ and prove a result, that can be used for arbitrary sampling and aggregation techniques.

\begin{proposition}\label{thm::constant-h-variance-general}
For $\varphi$ from (5) with $\varphi_s \perp \varphi_r$ for $r<s$ and $p_{s,r}\leq p_{s',r}$ for $s>s'$, i.e., the probability to sample a previous observation $r$ is non-increasing over time, it holds
\begin{equation*}
    \mathbb{V}\left[\tilde\phi_t^{(S)}\right] \leq 4\sigma_w^2\sigma_2^2 + 2\sigma_{2}^2\sum_{s=t_0}^t\sum_{s'=t_0}^{s-1} w_s w_{s'}\underbrace{\sum_{r=0}^{s'-1}p_{s',r}^2}_{=: \mathcal I_\varphi(s)},
\end{equation*}
provided that $\sigma_2^2 := \mathbb{V}[f(Z_s,Z_r)] <\infty$ and with $\sigma^{2}_w := \sum_{s=0}^t w_s^2$.
\end{proposition}

\begin{proof}
We denote $f(Z_s,Z_r) := \Vert h(X_s^{(\bar{S_j})},X_r^{(S)})-Y_s\Vert -\Vert h(X_s)-Y_s\Vert$.
Using $p_{s,r} := \mathbb{P}(\varphi_s=r)$ and properties of variance, we can write
\begin{align*}
\mathbb{V}[\tilde\phi_t^{}] = \mathbb{V}[\sum_{s=t-N+1}^t w_s \sum_{r=0}^{s-1} p_{r,s} f(Z_s,Z_r)]
\\
= \sum_{s,s'=t_0}^t w_s w_{s'} \sum_{r=0}^{s-1}\sum_{r'=0}^{s'-1} p_{s,r}p_{s',r'}\text{cov}((s,r),(s',r')),
\end{align*}
where $\text{cov}((s,r),(s',r')):= \text{cov}(f(Z_s,Z_r),f(Z_{s'},Z_{r'}))$ denotes the covariance of the two random variables.
The above sum ranges over all possible combinations of pairs $(s,r)$, where $s=t_0\dots,t$ and $r=0,\dots,s-1$.
As $r<s$ and $r'<s'$, it holds $\vert\{s,s',r,r'\}\vert \geq 2$.
When $\vert\{s,s',r,r'\}\vert=2$ then $s=s'$ and $r=r'$ and the covariance reduces to the variance.
When none of the indices match, i.e., $\vert\{s,s',r,r'\}\vert = 4$, then the covariance is zero, due to the independence assumption.
When exactly one index matches, then there are three possible cases:
\begin{itemize}
\item case 1: $s=s',r\neq r'$, 
\item case 2: $s\neq s',r\neq r$ with $r'=s$ or $s'=r$ 
\item case 3:  $s \neq s', r=r'$.
\end{itemize}

Case 2 yields the same covariances due to the iid assumption and the symmetric of the covariance.
For case 1, with $\mathbb{E}_{(Z_s,Z_r)}[f(Z_s,Z_r)] = \mathbb{E}_{Z_s} \mathbb{E}_{Z_r}[f(Z_s,Z_r)] = \phi^{(S)}(h)$, we denote $\tilde f(Z_s,Z_r) := f(Z_s,Z_r) - \phi^{(S)}(h)$ to compute the covariance as
\begin{align*}
   \text{cov}((s,r),(s',r')) &= \mathbb{E}[\tilde f(Z_{s},Z_{r})\tilde f(Z_{s},Z_{r'})]
   \\
   &=\mathbb{E}_{Z_{s}}[\mathbb{E}_{Z_{r}}[\tilde f(Z_{s},Z_{r})]\mathbb{E}_{Z_{r'}}[\tilde f(Z_{s},Z_{r'})]]
   \\
    &=\mathbb{E}_{Z_{s}}[\mathbb{E}_{Z_{r}}[\tilde f(Z_{s},Z_{r})]^2] 
    \\
    &= \mathbb{V}_{Z_{s}}[\mathbb{E}_{Z_{r}}[f(Z_{s},Z_{r})]],
\end{align*}
where we have used $\mathbb{E}_{Z_{s}}[\mathbb{E}_{Z_{r}}[\tilde f(Z_{s},Z_{r})]]=\phi^{(S)}(h)$ as well as the iid assumption multiple times, in particular when $\mathbb{E}_{Z_{r}}[f(Z_{s},Z_{r})]=\mathbb{E}_{Z_{r'}}[f(Z_{s},Z_{r'})]$.
The same arguments apply for the second argument for case 3, as
\begin{align*}
   \text{cov}((s,r),(s',r')) = \mathbb{V}_{Z_{r}}[\mathbb{E}_{Z_{s}}[f(Z_{s},Z_{r})]].
 \end{align*}
We thus summarize
\begin{equation*}
   \text{cov}((s,r),(s',r')) = 
   \begin{cases}
   \mathbb{V}[f(Z_s,Z_r)], \text{ if } s=s', r=r'
    \\
    \mathbb{V}_{Z_s}[\mathbb{E}_{Z_r}[f(Z_s,Z_r)]], \text{ if case 1}
    \\
   \text{cov}((s,r),(s',r')), \text{ if case 2}
    \\
    \mathbb{V}_{Z_r}[\mathbb{E}_{Z_s}[f(Z_s,Z_r)]], \text{ if case 3}
    \\
    0, \text{ if } \vert\{s,s',r,r'\}\vert = 4.
\end{cases}
\end{equation*}
By the Cauchy-Schwarz inequality all covariances are bounded by $\sigma_2^2 :=  \mathbb{V}[f(Z_s,Z_r)]$.
With $I := \{t_0,\dots,t\}$ and $I_{s} := \{0,\dots,s-1\}$ and $Q_2 := \{(s,r):s=s' \in I, r=r' \in I_s\}$ $Q_3:= \{(s,s',r,r') : s,s' \in I, r \in I_s, r' \in I_{r'}, \vert\{s,s',r,r'\}\vert=3\}$.
We thus obtain
\begin{align*}
\mathbb{V}[\tilde\phi_t^{(S)}] &= \sigma_2^2\sum_{(s,r) \in Q_2}w_s^2p_{s,r}^2
\\
&+ \sum_{(s,s',r,r') \in Q_3}w_s w_{s'}p_{s,r}p_{s',r'}\text{cov}((s,r),(s',r')).
\end{align*}
For the first sum, we have
\begin{equation*}
    \sum_{(s,r) \in Q_2}w_s^2 p_{s,r}^2 \leq \sum_{(s,r) \in Q_2}w_s^2 p_{s,r} = \sum_{s=t_0}^t w_s^2 = \sigma_w^2.
\end{equation*}
For the second sum, $Q_3$ decomposes into the three cases. 
For case 1,
\begin{align*}
    \sum_{\substack{(s,s',r,r') \in Q_3 \\ s=s', r\neq r'}} w_s w_{s'} p_{s,r}p_{s,r'} &= \sum_{s=t_0}^t w_s w_{s'}  \sum_{\substack{(r,r')\in I_s^2 \\ r\neq r'}} p_{s,r}p_{s,r'} 
    \\
    &\leq \sum_{s=t_0}^t w_s^2(\sum_{r=0}^{s-1}p_{s,r})^2 = \sigma_w^2.
\end{align*}
For case 2 w.l.o.g assume $r=s'$, which implies $s>s'$ and thus $w_s\geq w_{s'}$, then
\begin{align*}
   \sum_{\substack{(s,s',r,r') \in Q_3 \\ s \neq s', r\neq r', s'=r}}w_s w_{s'} p_{s,s'}p_{s',r'} &= \sum_{s =t_0}^t w_s \sum_{s'=t_0}^{s-1} w_{s'}p_{s,s'}
   \\
   &\leq \sum_{s =t_0}^t w_s^2 = \sigma_w^2.
\end{align*}
For case 3, we have
\begin{align*}
\sum_{\substack{(s,s',r,r') \in Q_3 \\ s \neq s', r=r'}}w_s w_{s'}p_{s,r}p_{s',r}  &=\sum_{\substack{(s,s')\in I^2 \\ s\neq s'}}w_s w_{s'}\sum_{r=0}^{\min(s,s')-1}p_{s,r}p_{s',r} 
\\
&= 2\sum_{\substack{(s,s')\in I^2 \\ s> s'}}w_s w_{s'}\sum_{r=0}^{s'-1}p_{s,r}p_{s',r} 
\\
&\leq 2\sum_{\substack{(s,s')\in I^2 \\ s> s'}}w_s w_{s'}\sum_{r=0}^{s'-1}p_{s',r}^2.
\end{align*}
In summary, we conclude
\begin{align*}
    \mathbb{V}\left[\tilde\phi_t^{(S)}\right] \leq 4\sigma_w^2\sigma_2^2 + 2\sigma_{2}^2\sum_{s=t_0}^t\sum_{s'=t_0}^{s-1} w_s w_{s'}\sum_{r=0}^{s'-1}p_{s',r}^2.
\end{align*}
\end{proof}

The last sum depends on both the choices of weights $w_s$ and the \emph{collision probability} $\mathcal I_\varphi(s) = \sum_{r=0}^{s-1} p_{s,r}^2 = P(Q_1=Q_2)$ for $Q_1,Q_2 \overset{iid}{\sim} \mathbb{P}_{\varphi_s}$, which is related to the Rényi entropy \cite{renyi1961measures}.
The variance increases with the collision probabilities of the sampling strategy, in particular $\mathcal I_{\text{unif}}(s) = \frac 1 s$ and $\mathcal I_{\text{geom}}(s) = \frac{p}{2-p} (1 + (1-p)^{2(s-t_0)+1})$ for uniform and geometric sampling, respectively.

\begin{lemma}
For geometric sampling and $p \in (0,1)$ it holds
\begin{equation*}
    \mathcal I_{\text{geom}}(s) = \sum_{r=0}^{s-1}p^2_{s,r} 
    =\frac p {2-p}(1 + (1-p)^{2(s-t_0)+1}).
\end{equation*}
\end{lemma}

\begin{proof}
The probabilities for geometric sampling are
\begin{equation*}
    p_{s,r} = \begin{cases}
      p \cdot (1-p)^{s-r-1}, r>t_0= \frac 1 p
      \\
      p \cdot (1-p)^{s-t_0}, r\leq t_0= \frac 1 p.
    \end{cases}
\end{equation*}
Then
\begin{align*}
    \mathcal I_{\text{geom}}(s) &= \sum_{r=0}^{s-1}p^2_{s,r} 
    \\
    &= \sum_{r=0}^{t_0-1} p^2 \cdot (1-p_r)^{2(s-t_0)} + \sum_{r=t_0}^{s-1} p^2 (1-p)^{2(s-r-1)}
    \\
    &=t_0 \cdot p^2 \cdot (1-p)^{2(s-t_0)} + \sum_{r=t_0}^{s-1} p^2 (1-p)^{2(s-r-1)}
    \\
    &= p\cdot (1-p)^{2(s-t_0)} + p^2\sum_{r=0}^{s-t_0-1} (1-p)^{2r}
    \\
    &= p \cdot (1-p)^{2(s-t_0)} + p^2 \frac{1-(1-p)^{2(s-t_0)}}{1-(1-p)^2}
    \\
    &=  p \cdot (1-p)^{2(s-t_0)} + \frac p {2-p} (1-(1-p)^{2(s-t_0)})
    \\
    &= \frac p {2-p}(1 + (1-p)^{2(s-t_0)+1}).
\end{align*}
\end{proof}

We now apply Proposition \ref{thm::constant-h-variance-general} to our particular estimator $\bar\phi^{(S)} = \tilde\phi^{(S)}$ with $w_s := \alpha(1-\alpha)^{t-s}$ and take the limit for $t\to\infty$. 
Note that both uniform and geometric sampling fulfill the condition of the theorem.
Furthermore, we have $\sigma^2_w = \alpha^2 \sum_{s=0}^{t-t_0}(1-\alpha)^s \nearrow \frac \alpha {2-\alpha}$.

\paragraph{Uniform Sampling}
For uniform sampling, we have
\begin{align*}
\mathbb{V}[\bar\phi_t^{(S)}] 
&\leq\frac \alpha {2-\alpha}4\sigma_2^2 + 2\sigma_{2}^2\sum_{s=t_0}^t\sum_{s'=t_0}^{s-1} \alpha^2 \frac{(1-\alpha)^{t-s+t-s'}}{s'}
\\
&\leq \frac \alpha {2-\alpha}4\sigma_2^2 + 2\sigma_{2}^2\alpha^2\sum_{s=0}^{t-t_0} (1-\alpha)^s\sum_{s'=0}^{t-t_0} \frac{(1-\alpha)^{s'}}{t-s'}
\end{align*}
For the first sum, we have $\alpha \sum_{s=0}^{t-t_0} (1-\alpha)^s \nearrow  1$ for $t\to \infty$. For the second sum
\begin{align*}
 \alpha\sum_{s'=0}^{t-t_0} \frac{(1-\alpha)^{s'}}{t-s'} &\leq \alpha(\sum_{\substack{s'=0 \\ s' \geq t/2}}^{t-t_0} (1-\alpha)^{s'} + 1+  \sum_{\substack{s'=1 \\ s' < t/2}}^{t-t_0} \frac{(1-\alpha)^{s'}}{s'})
 \\
 &\leq (1-\alpha)^{t/2}-(1-\alpha)^{t-t_0+1} + \alpha -\alpha\log(\alpha)
 \\
 &\overset{t \to \infty}{\longrightarrow} \alpha - \alpha \log(\alpha).
\end{align*}
Hence,
\begin{equation*}
    \mathbb{V}[\lim_{t\to\infty}\bar\phi_t^{(S)}] = \mathcal O (-\alpha\log(\alpha)).
\end{equation*}

\paragraph{Geometric Sampling}
For geometric sampling, we have
\begin{align*}
    \mathbb{V}[\bar\phi_t^{(S)}] &\leq \underbrace{\frac \alpha {2-\alpha}4\sigma_2^2}_{= \mathcal O(\alpha)}
    \\
    &+ 2\sigma_{2}^2 \underbrace{\alpha^2 \sum_{s=t_0}^t \sum_{s'=t_0}^{s-1} (1-\alpha)^{t-s+t-s'}]}_{=: q(\alpha)}\mathcal I_{\text{geom}}(s).
\end{align*}
For the second term it is enough to show that $0<\lim_{t\to\infty} q(\alpha)<\infty$ to prove the result, as $\mathcal I_{\text{geom}}(s) = \mathcal O(p)$.
By using the properties of geometric progression, we obtain
\begin{align*}
    q(\alpha)
    &= \alpha \sum_{s=t_0}^t (1-\alpha)^{t-s} \alpha\sum_{s'=t-s}^{t-t_0}(1-\alpha)^{s'} 
    \\
    &=\alpha \sum_{s=t_0}^t (1-\alpha)^{t-s} ((1-\alpha)^{t-s}-(1-\alpha)^{t-t_0+1})
    \\
    &=\alpha \sum_{s=0}^{t-t_0} (1-\alpha)^{s} ((1-\alpha)^{s}-(1-\alpha)^{t-t_0+1})
    \\
    &= \underbrace{\alpha \sum_{s=0}^{t-t_0} (1-\alpha)^{2s}}_{\nearrow \frac 1 {2-\alpha}} - (1-\alpha)^{t-t_0+1}\underbrace{\alpha \sum_{s=0}^{t-t_0}(1-\alpha)^{s}}_{\nearrow 1}
    \\
    &\overset{t \to \infty}{\longrightarrow} \frac 1 {2-\alpha}.
\end{align*}
Hence,
\begin{equation*}
    \mathbb{V}[\lim_{t\to\infty}\bar\phi_t^{(S)}] \leq \mathcal O(\alpha) + 2 \sigma_2^2 \frac 2 {2-\alpha} \frac p {2-p} = \mathcal O(\alpha) + \mathcal O(p).
\end{equation*}

\end{proof}

\begin{theorem}[Bias for changing Model]\label{thm::changing-h-bias}
If $\Delta(h_s,h_t) \leq \delta$ and $\Delta_S(h_s,h_t) \leq \delta_S$ for $t_0 \leq s \leq t$, then
\begin{equation*}
 \vert \bar\phi^{(S_j)}_t - \phi^{(S_j)}(h_t)\vert \leq \delta_S + \delta +\mathcal O((1-\alpha)^{t}).   
\end{equation*}
\end{theorem}

\begin{proof}
We again consider the more general estimator $\tilde \phi^{(S)}_t := \mathbb{E}_{\varphi}[\sum_{s=t_0}^t w_s \hat\lambda^{(S)}_t(x_t,x_{\varphi_t},y_t)]$ and prove a more general result.

\begin{proposition}
If $\Delta(h_s,h_t) \leq \delta$ and $\Delta_S(h_s,h_t) \leq \delta_S$ for $t_0 \leq s \leq t$, then
$\vert \mathbb{E}[\hat\phi^{(S)}_t] - \phi^{(S)}(h_t)\vert \leq \mu_w (\delta_S + \delta) + \vert(1-\mu_w) \phi^{(S)}(h_t)\vert$.
\end{proposition}

\begin{proof}
For the proof, we first show that for two models $h_s,h_t$ and a subset $S \subset D$, it holds that
$\vert\phi^{(S)}(h_t)-\phi^{(S)}(h_s)\vert \leq \Delta_S(h_s,h_t) + \Delta(h_s,h_t)$.
This follows directly from the reverse triangle inequality for $f_S^\Delta(x^{(\bar S)},h_s,h_t) \geq \mathbb{E}_{\tilde X}[ \Vert h_t(x^{(\bar S)},\tilde X)-y\Vert - \Vert y-h_s(x^{(\bar S)},\tilde X)\Vert]$.
The result then follows directly by definition, the observation that $\hat\lambda_s^{(S)}$ is an unbiased estimate of $\phi^{(S)}(h_s)$, as
\begin{align*}
    \vert \mathbb{E}[\bar\phi^{(S)}_t] - \phi^{(S)}(h_t) \vert 
    &= \vert (\sum_{s=t_0}^t w_s \phi^{(S)}(h_s)) - \phi^{(S)}(h_t)\vert
    \\
    &\leq \sum_{s=t_0}^t w_s\underbrace{\vert \phi^{(S)}(h_s) - \phi^{(S)}(h_t)\vert}_{\leq \delta + \delta_S} 
    \\
    &+ \vert(\sum_{s=t_0}^t w_s-1)\phi^{(S)}(h_t) \vert
    \\
    &\leq \mu_w(\delta+\delta_S) + \underbrace{\vert(1-\mu_w)\phi^{(S)}(h_t)\vert}_{\text{bias for static model}}.
\end{align*}
\end{proof}
With $\mu_w = 1 - (1-\alpha)^{t-t_0+1}$ our special case follows immediately.
\end{proof}

\begin{theorem}[Variance for changing Model]\label{thm::changing-h-variance}
If
\begin{equation}\label{eq::bound-variability}
    \text{cov}(f_s(Z_s,Z_r),f_{s'}(Z_{s'},Z_{r'})) 
    \leq \sigma_{\text{max}}^2
\end{equation}
for $t_0\leq s,s' \leq t$, $r<s$ and $r'<s'$, then for a sequence of models $(h_t)_{t\geq0}$ the results of Theorem \ref{thm::constant-h-variance} apply.
\end{theorem}

\begin{proof}
In all proofs a changing model $h_t$ adds a time dependency on the function $f_s(Z_s,Z_r) := \Vert h_s(X_s^{(\bar S)},X_r^{(S)})-Y_s\Vert -\Vert h_s(X_s)-Y_s\Vert$. 
Instead of bounding the covariances by $\sigma_2^2$, we now bound the covariances of the time-dependent functions by $\sigma_{\text{max}}^2$.
This only directly affects Proposition \ref{thm::constant-h-variance-general}, as
\begin{align*}
\mathbb{V}[\bar\phi_t^{(S)}] &= \mathbb{V}[\sum_{s=t-N+1}^t w_s \sum_{r=0}^{s-1} p_{r,s} f_s(Z_s,Z_r)]
\\
&= \sum_{s,s'=t_0}^t w_s w_{s'} \sum_{r=0}^{s-1}\sum_{r'=0}^{s'-1} p_{s,r}p_{s',r'} \text{cov}((s,r),(s',r'))
\\
&\leq \sigma^2_{\text{max}}\sum_{s,s'=t_0}^t w_s w_{s'} \sum_{r=0}^{s-1}\sum_{r'=0}^{s'-1} p_{s,r}p_{s',r'}.
\end{align*}
All remaining arguments and proofs are still valid for a changing model due to the iid assumption.
\end{proof}

\section{Link to FI used in SAGE} 
Our definition of FI aligns with the definition of FI $v_h(S)$ given by \citet{Covert_Lundberg_Lee_2020}, which measures the \emph{increase in risk} when \emph{including} features in $S$ compared with marginalizing all features.
The conditional distribution $X^{(S)} \vert X^{(\bar S)}$ coincides with the marginal distribution if $X^{(\bar S)} \perp X^{(S)}$, which is often assumed in practice \cite{Covert_Lundberg_Lee_2021,Covert_Lundberg_Lee_2020,NIPS2017_7062}.
It was even suggested that the marginal distribution is conceptually the right choice \cite{pmlr-v108-janzing20a}.

\section{Approximation Error for expected PFI}
With $f(Z_n,Z_m) := \Vert h(X_n^{(\bar{S_j})},X_m^{(S_j)})-Y_n\Vert -\Vert h(X_n)-Y_n\Vert$ and symmetric U-statistic kernel $f_0(Z_n,Z_m):= \frac{f(Z_n,Z_m)+f(Z_m,Z_n)}{2}$, we can write
\begin{equation*}
    \bar\phi^{(S_j)} = \binom{N}{2}^{-1} \sum_{1\leq n < m \leq N} f_0(Z_n,Z_m),  
\end{equation*}
which is the basic form of a U-statistic and therefore the variance can be computed as
\begin{equation*}
\mathbb{V}\left[\bar\phi^{(S_j)}\right]= \binom{N}{2}^{-1}\sum_{c=1}^2 \binom{2}{c}\binom{N-2}{2-c}\sigma^2_c = \mathcal{O}(1/N),
\end{equation*}
where $\sigma_1^2 := \mathbb{V}_{Z_n}[\mathbb{E}_{Z_m}[f_0(Z_n,Z_m)]]$ and $\sigma_2^2 := \mathbb{V}[f_0(Z_n,Z_m)]$ are assumed to be finite \cite{10.1214/aoms/1177730196_hoeff_u}.
For $\epsilon > 0$, we then obtain by Chebyshev's inequality $ \mathbb{P}(\vert \bar\phi^{(S_j)} - \phi^{(S_j)}(h)\vert > \epsilon) = \mathcal{O}(1/N)$, as $\bar\phi^{(S_j)}$ is unbiased.

\section{Ground-truth PFI for the \emph{agrawal} stream.}

River \cite{montiel_river_2020} implements the \emph{agrawal} \cite{Agrawal.1993} data stream with multiple classification functions.
In our experiments we consider the following classification function (among others):
\begin{align*}
    \text{Class A:\ } & ((\text{age} < 40) \land (50K \leq \text{salary} \leq 100K ))\ \lor \\
    & ((40 \leq \text{age} < 60) \land (75K \leq \text{salary} \leq 125K ))\ \lor \\
    & ((\text{age} \geq 60) \land (25K \leq \text{salary} \leq 75K ))
\end{align*}
Both feature \emph{age} and \emph{salary} are uniformly distributed with $X^{(\text{age})} \sim \mathcal{U}_{[20, 80]}$ and $X^{(\text{salary})} \sim \mathcal{U}_{[20, 150]}$.
Given iid. samples from the data stream the classification problem can be transformed into a two-dimensional problem following the above defined classification function.
The two-dimensional classification problem is illustrated in Figure.~\ref{fig:theoretical_pfi}.
A sample is classified as concept $A$ when it occurs contained in $A_1$, $A_2$, or $A_3$.
Otherwise the sample is classified as concept $B$.

The theoretical PFIs can be calculated with the base probability of an sample belonging to concept A ($P(A_1) = P(A_2) = P(A_3) = \frac{5}{39}$) times the probability of switching the class through changing a feature ($P(A_i \rightarrow B_{n,m})$) plus the vice versa for a sample originally belonging to concept $B$.

{
\small
\begin{align*}
    \phi^{(\text{age})} &= P(A_1) \cdot P(A_1 \rightarrow B_{11}) + P(B_{11}) \cdot P(B_{11} \rightarrow A_1)
    \\
    &+ P(A_2) \cdot P(A_2 \rightarrow B_{21}) + P(B_{21}) \cdot P(B_{21} \rightarrow A_2)
    \\
    &+ P(A_3) \cdot P(A_3 \rightarrow B_{31}) + P(B_{31}) \cdot P(B_{31} \rightarrow A_3) =
    \\
    &= \frac{5}{39} \cdot \frac{1}{3} + (\frac{5}{13} \cdot \frac{1}{3} ) \cdot \frac{1}{3}
    \\
    &+ 2 \cdot (\frac{5}{39} \cdot \frac{1}{2} + (\frac{5}{13}\cdot\frac{1}{3}+\frac{5}{13}\cdot\frac{1}{3}\cdot\frac{1}{2})\cdot\frac{1}{3}) \approx
    \\
    &\approx 0.3419
    \\
    \phi^{(\text{salary})} &= P(A_1) \cdot P(A_1 \rightarrow B_{12}) + P(B_{12}) \cdot P(B_{12} \rightarrow A_1)
    \\
    &+ P(A_2) \cdot P(A_2 \rightarrow B_{22}) + P(B_{22}) \cdot P(B_{22} \rightarrow A_2)
    \\
    &+ P(A_3) \cdot P(A_3 \rightarrow B_{32}) + P(B_{32}) \cdot P(B_{32} \rightarrow A_3) =
    \\
    &= 3 \cdot (\frac{5}{39} \cdot \frac{8}{13} + (\frac{8}{13} \cdot \frac{1}{3}) \cdot \frac{5}{13})\approx\\
    &\approx 0.4734
\end{align*}
}

\begin{figure}[h]
    \centering
    \includegraphics[width=0.9\columnwidth]{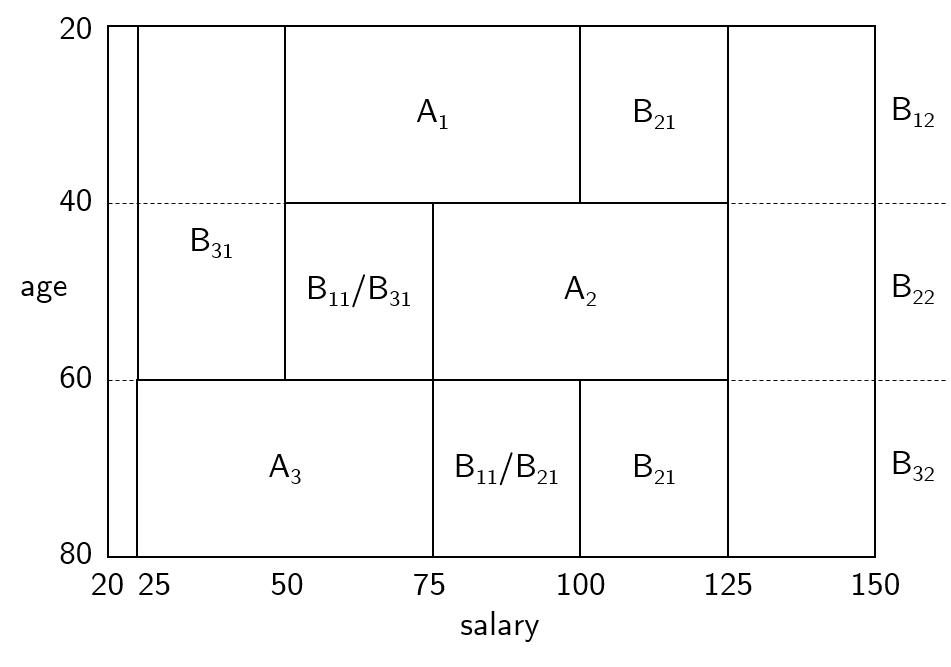}
    \caption{Two-Dimensional Classification Problem of the \emph{agrawal} Data Stream.}
    \label{fig:theoretical_pfi}
\end{figure}

\section{Experiments}
In the following, we give more comprehensive details about the datasets and models used in our experiments.

\subsection{Dataset Description}
\paragraph{adult \cite{adult_1996}}
Binary classification dataset that classifies 48842 individuals based on 14 features into yearly salaries above and below 50k.
There are six numerical features and eight nominal features.

\paragraph{bank \cite{moro2011using}}
Binary classification dataset that classifies 45211 marketing phone calls based on 17 features to decide whether they decided to subscribe a term deposit.
There are seven numerical features and ten nominal features.

\paragraph{bike \cite{fanaee2014event}}
Regression dataset that collects the number of bikes in different bike stations of Toulouse over 187470 time stamps.
There are six numerical features and two nominal features.

\paragraph{elec2 \cite{Harries99splice-2comparative}}
Binary classification dataset that classifies, if the electricity price will go up or down.
The data was collected for 45312 time stamp from the Australian New South Wales Electricity Market and is based on eight features, six numerical and two nominal.

\paragraph{agrawal \cite{Agrawal.1993}}
Synthetic data stream generator to create binary classification problems to decide whether an indivdual will be granted a loan based on nine features, six numerical and three nominal.
There are ten different decision functions available.

\paragraph{stagger \cite{DBLP:journals/ml/SchlimmerG86}}
The \emph{stagger} concepts makes a simple toy classification data stream. The syntethtical data stream generator consists of three independent categorical features that describe the \emph{shape}, \emph{size}, and \emph{color} of an artificial object. Different classification functions can be derived from these sharp distinctions.

\subsection{Model Description}
All models are implemented with the default parameters from \textit{scikit-learn} \cite{scikit-learn} and \textit{River} \cite{montiel_river_2020} unless otherwise stated.

\paragraph{ARF}
The Adaptive Random Forest Classifier (ARF) uses an ensemble of 50 trees with binary splits, ADWIN drift detection and information gain split criterion.
We used the default implementation \emph{AdaptiveRandomForestClassifier} from River with n\_models=50 and binary\_split=True.

\paragraph{NN}
The Neural Network classifier (NN) was implemented with two hidden layers of size $128 \times 64$, ReLu activation function and optimized with stochastic gradient descent (ADAM). We used the default implementation \emph{MLPClassifier} from scikit-learn.

\paragraph{GBT}
The Gradient Boosting Tree (GBT) uses 200 estimators and additively builds a decision tree ensemble using log-loss optimization.
We used the GradientBoostingClassifier from scikit-learn with n\_estimators=200.

\paragraph{LGBM}
The LightGBM (LGBM) constitutes a more lightweight implementation of GBT.
We used HistGradientBoostingRegressor for regression tasks and HistGradientBoostingClassifier for classification tasks from scikit-learn with the standard parameters.

\subsection{Hardware Details}
The experiments were mainly run on an computation cluster on hyperthreaded Intel Xeon E5-2697 v3 CPUs clocking at with 2.6Ghz.
In total the experiments took around 300 CPU hours (30 CPUs for 10 hours) on the cluster.
This mainly stems from the number of parameters and different initializations. 
Before running the experiments on the cluster, the implementations were validated on a Dell XPS 15 9510 containing an Intel i7-11800H at 2.30GHz. 
The laptop was running for around 12 hours for this validation. 

\subsection{Summary of experiments}
Table~\ref{tab:summary_conept_drift_appendix} contains summary information about the supplementary experiments. 
Table~\ref{tab:summary_conept_drift_main} contains additional information for Experiment B and C in the main contribution.

\begin{table*}[]
\centering
{
\setlength{\tabcolsep}{3pt}
\renewcommand{\arraystretch}{1.25}
\small
\begin{tabular}{cccc|cccc|cccc|cccc}
\toprule
\multirow{2}{*}{\textbf{data}} & \multirow{2}{*}{\textbf{\begin{tabular}[c]{@{}c@{}}image\\ id\end{tabular}}} & \textbf{sampling} & \multirow{2}{*}{$\boldsymbol{\alpha}$} & \multicolumn{4}{c}{\textbf{whole stream}} & \multicolumn{4}{c}{\textbf{before drift}} & \multicolumn{4}{c}{\textbf{after drift}} \\
 &  & \textbf{strategy} &  & $\boldsymbol{Q_{2}}$ & \textbf{IQR} & $\boldsymbol{Q_{1}}$ & $\boldsymbol{Q_{3}}$ & $\boldsymbol{Q_{2}}$ & \textbf{IQR} & $\boldsymbol{Q_{1}}$ & $\boldsymbol{Q_{3}}$ & $\boldsymbol{Q_{2}}$ & \textbf{IQR} & $\boldsymbol{Q_{1}}$ & $\boldsymbol{Q_{3}}$ \\ \midrule
\multirow{24}{*}{\emph{agrawal}} & \multirow{4}{*}{fu. 1} & uniform & 0.001 & 0.050 & 0.054 & 0.025 & 0.079 & 0.075 & 0.018 & 0.063 & 0.080 & 0.024 & 0.028 & 0.009 & 0.038 \\
 &  & geometric & 0.001 & 0.052 & 0.060 & 0.024 & 0.084 & 0.071 & 0.020 & 0.068 & 0.088 & 0.022 & 0.023 & 0.014 & 0.037 \\
 &  & uniform & 0.01 & 0.047 & 0.075 & 0.021 & 0.096 & 0.098 & 0.020 & 0.091 & 0.111 & 0.018 & 0.011 & 0.017 & 0.029 \\
 &  & geometric & 0.01 & 0.040 & 0.080 & 0.027 & 0.107 & 0.116 & 0.044 & 0.078 & 0.122 & 0.027 & 0.008 & 0.021 & 0.029 \\ \cline{2-16} 
 & \multirow{4}{*}{fu. 2} & uniform & 0.001 & 0.067 & 0.060 & 0.050 & 0.110 & 0.064 & 0.023 & 0.047 & 0.070 & 0.072 & 0.065 & 0.058 & 0.123 \\
 &  & geometric & 0.001 & 0.063 & 0.066 & 0.044 & 0.110 & 0.059 & 0.026 & 0.041 & 0.067 & 0.074 & 0.071 & 0.051 & 0.122 \\
 &  & uniform & 0.01 & 0.111 & 0.135 & 0.061 & 0.196 & 0.208 & 0.088 & 0.153 & 0.240 & 0.058 & 0.016 & 0.052 & 0.069 \\
 &  & geometric & 0.01 & 0.103 & 0.088 & 0.071 & 0.159 & 0.166 & 0.101 & 0.140 & 0.241 & 0.070 & 0.028 & 0.059 & 0.087 \\ \cline{2-16} 
 & \multirow{4}{*}{fu. 2, early} & uniform & 0.001 & 0.067 & 0.123 & 0.035 & 0.158 & 0.110 & 0.060 & 0.080 & 0.140 & 0.064 & 0.105 & 0.032 & 0.137 \\
 &  & geometric & 0.001 & 0.066 & 0.132 & 0.036 & 0.168 & 0.116 & 0.067 & 0.082 & 0.149 & 0.063 & 0.106 & 0.032 & 0.138 \\
 &  & uniform & 0.01 & 0.069 & 0.113 & 0.052 & 0.165 & 0.217 & 0.043 & 0.195 & 0.238 & 0.066 & 0.042 & 0.046 & 0.088 \\
 &  & geometric & 0.01 & 0.078 & 0.103 & 0.055 & 0.157 & 0.187 & 0.020 & 0.177 & 0.196 & 0.069 & 0.042 & 0.050 & 0.092 \\ \cline{2-16} 
 & \multirow{4}{*}{fu. 2, late} & uniform & 0.001 & 0.071 & 0.106 & 0.045 & 0.151 & 0.051 & 0.031 & 0.042 & 0.072 & 0.244 & 0.231 & 0.163 & 0.394 \\
 &  & geometric & 0.001 & 0.081 & 0.105 & 0.051 & 0.156 & 0.061 & 0.042 & 0.041 & 0.082 & 0.246 & 0.230 & 0.170 & 0.400 \\
 &  & uniform & 0.01 & 0.117 & 0.093 & 0.066 & 0.159 & 0.139 & 0.087 & 0.069 & 0.156 & 0.095 & 0.091 & 0.071 & 0.162 \\
 &  & geometric & 0.01 & 0.103 & 0.115 & 0.063 & 0.178 & 0.128 & 0.106 & 0.065 & 0.170 & 0.077 & 0.101 & 0.063 & 0.163 \\ \cline{2-16} 
 & \multirow{4}{*}{fu. 3} & uniform & 0.001 & 0.079 & 0.071 & 0.037 & 0.108 & 0.097 & 0.026 & 0.086 & 0.111 & 0.032 & 0.037 & 0.016 & 0.053 \\
 &  & geometric & 0.001 & 0.081 & 0.078 & 0.035 & 0.113 & 0.095 & 0.036 & 0.084 & 0.119 & 0.029 & 0.036 & 0.017 & 0.053 \\
 &  & uniform & 0.01 & 0.097 & 0.108 & 0.053 & 0.161 & 0.149 & 0.044 & 0.134 & 0.178 & 0.056 & 0.009 & 0.051 & 0.060 \\
 &  & geometric & 0.01 & 0.124 & 0.067 & 0.087 & 0.153 & 0.142 & 0.022 & 0.135 & 0.157 & 0.090 & 0.038 & 0.074 & 0.112 \\ \cline{2-16} 
 & \multirow{4}{*}{fe. 1} & uniform & 0.001 & 0.048 & 0.102 & 0.018 & 0.121 & 0.021 & 0.036 & 0.017 & 0.054 & 0.087 & 0.177 & 0.043 & 0.220 \\
 &  & geometric & 0.001 & 0.035 & 0.091 & 0.015 & 0.106 & 0.023 & 0.031 & 0.018 & 0.049 & 0.047 & 0.117 & 0.008 & 0.125 \\
 &  & uniform & 0.01 & 0.062 & 0.044 & 0.034 & 0.077 & 0.072 & 0.023 & 0.056 & 0.079 & 0.046 & 0.037 & 0.030 & 0.067 \\
 &  & geometric & 0.01 & 0.044 & 0.052 & 0.035 & 0.087 & 0.079 & 0.047 & 0.042 & 0.089 & 0.043 & 0.014 & 0.032 & 0.046 \\ \midrule
\multirow{2}{*}{\emph{stagger}} & \multirow{2}{*}{fu. 1} & uniform & 0.001 & 0.018 & 0.118 & 0.014 & 0.132 & 0.009 & 0.005 & 0.007 & 0.012 & 0.132 & 0.443 & 0.075 & 0.518 \\
 &  & geometric & 0.001 & 0.018 & 0.117 & 0.015 & 0.131 & 0.008 & 0.006 & 0.005 & 0.011 & 0.131 & 0.440 & 0.075 & 0.515 \\ \midrule
\multirow{4}{*}{\emph{elec2}} & \multirow{2}{*}{fe. 1} & uniform & 0.001 & 0.270 & 0.305 & 0.042 & 0.347 & 0.041 & 0.061 & 0.033 & 0.093 & 0.353 & 0.068 & 0.311 & 0.378 \\
 &  & geometric & 0.001 & 0.037 & 0.039 & 0.033 & 0.072 & 0.037 & 0.066 & 0.032 & 0.098 & 0.037 & 0.022 & 0.036 & 0.057 \\ \cline{2-16} 
 & \multirow{2}{*}{fe. 1, gradual} & uniform & 0.001 & 0.158 & 0.263 & 0.050 & 0.313 & 0.048 & 0.075 & 0.025 & 0.101 & 0.321 & 0.089 & 0.283 & 0.372 \\
 &  & geometric & 0.001 & 0.037 & 0.024 & 0.027 & 0.051 & 0.040 & 0.069 & 0.026 & 0.095 & 0.037 & 0.013 & 0.028 & 0.041 \\ \bottomrule
\end{tabular}
\caption{Summary of additional concept drift experiments on \emph{agrawal}, \emph{stagger}, and \emph{elec2}. The image identifier point to the subsequent section of figures. $Q_2$ denotes the median of the error described in Experiment A computed for iPFI and interval PFI (solid line vs. dashed line in the Figures). The interquartile range is calculated between $Q_1$ and $Q_3$.}
\label{tab:summary_conept_drift_appendix}
}
\end{table*}

{
\setlength{\tabcolsep}{3pt}
\renewcommand{\arraystretch}{1.25}
\small
\begin{table*}[]
\centering
\begin{tabular}{@{}ccc|cccc|cccc|cccc@{}}
\toprule
\multirow{2}{*}{\textbf{data}} & \multirow{2}{*}{\textbf{exp.}} & \textbf{image} & \multicolumn{4}{c|}{\textbf{whole stream}} & \multicolumn{4}{c|}{\textbf{before drift}} & \multicolumn{4}{c}{\textbf{after drift}} \\
 &  & \textbf{facet} & $\boldsymbol{Q_{2}}$ & \textbf{IQR} & $\boldsymbol{Q_{1}}$ & $\boldsymbol{Q_{3}}$ & $\boldsymbol{Q_{2}}$ & \textbf{IQR} & $\boldsymbol{Q_{1}}$ & $\boldsymbol{Q_{3}}$ & $\boldsymbol{Q_{2}}$ & \textbf{IQR} & $\boldsymbol{Q_{1}}$ & $\boldsymbol{Q_{3}}$ \\ \midrule
\multirow{2}{*}{agrawal} & \multirow{2}{*}{B} & \begin{tabular}{c}function-\\drift\end{tabular} & 0.052 & 0.060 & 0.024 & 0.084 & 0.071 & 0.020 & 0.068 & 0.088 & 0.022 & 0.023 & 0.014 & 0.037 \\
 &  &\begin{tabular}{c}feature-\\drift\end{tabular}  & 0.035 & 0.091 & 0.015 & 0.106 & 0.023 & 0.031 & 0.018 & 0.049 & 0.047 & 0.117 & 0.008 & 0.125 \\ \midrule
\multirow{2}{*}{elec2} & \multirow{2}{*}{C} & \begin{tabular}{c}uniform\\sampling\end{tabular} & 0.270 & 0.305 & 0.042 & 0.347 & 0.041 & 0.061 & 0.033 & 0.093 & 0.353 & 0.068 & 0.311 & 0.378 \\
 &  & \begin{tabular}{c}geometric\\sampling\end{tabular} & 0.037 & 0.039 & 0.033 & 0.072 & 0.037 & 0.066 & 0.032 & 0.098 & 0.037 & 0.022 & 0.036 & 0.057 \\ \bottomrule
\end{tabular}
\caption{Summary of experiment B and C's iPFI's error against interval PFI (solid line vs. dashed line in the Figures).}
\label{tab:summary_conept_drift_main}
\end{table*}
}

\begin{figure*}[!th]
\begin{minipage}{\textwidth}
    \begin{minipage}[b]{0.49\textwidth}
        \centering
        \includegraphics[width=1\textwidth]{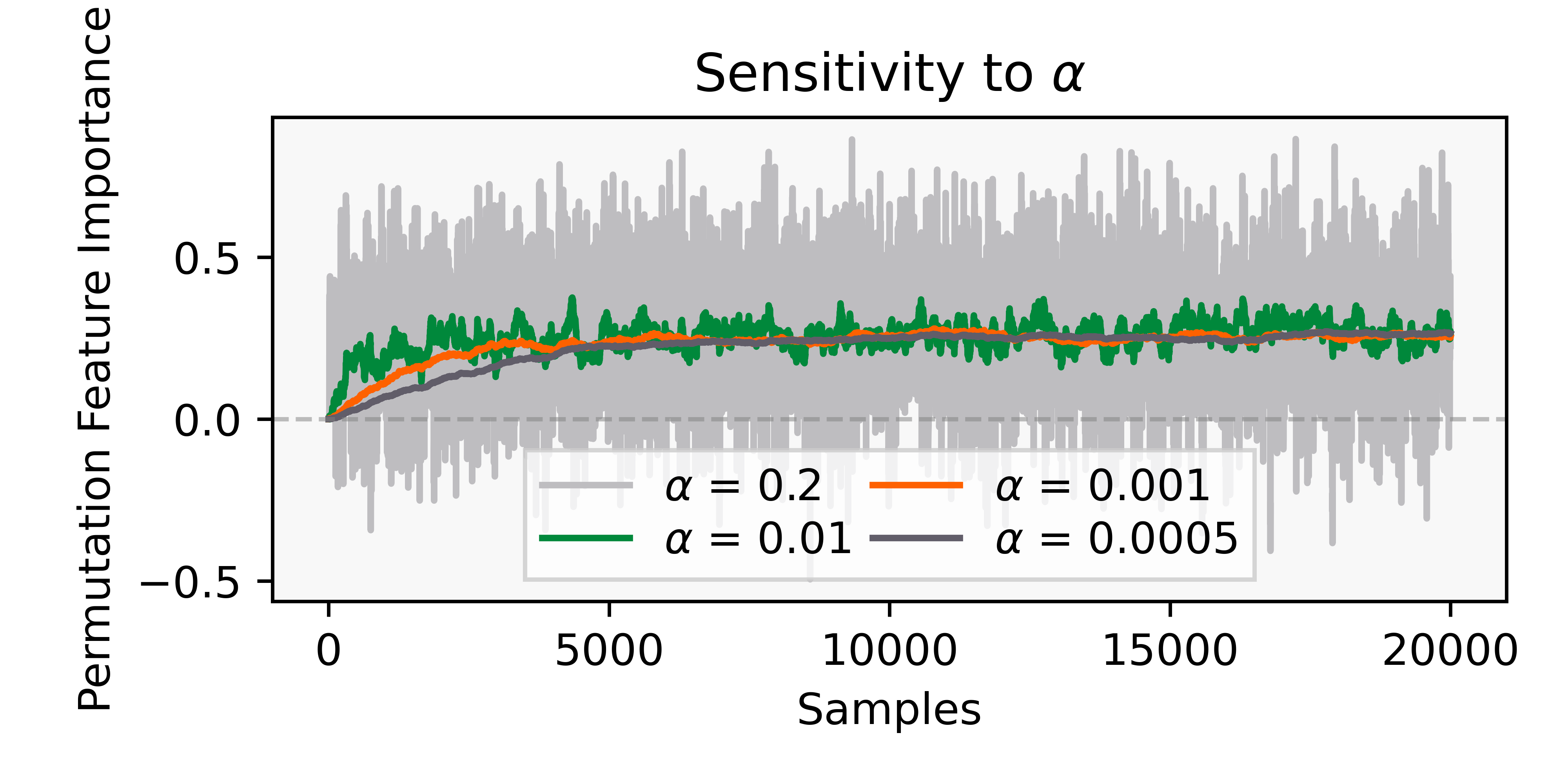}
    \end{minipage}
    \hfill
    \begin{minipage}[b]{0.49\textwidth}
        \includegraphics[width=1\textwidth]{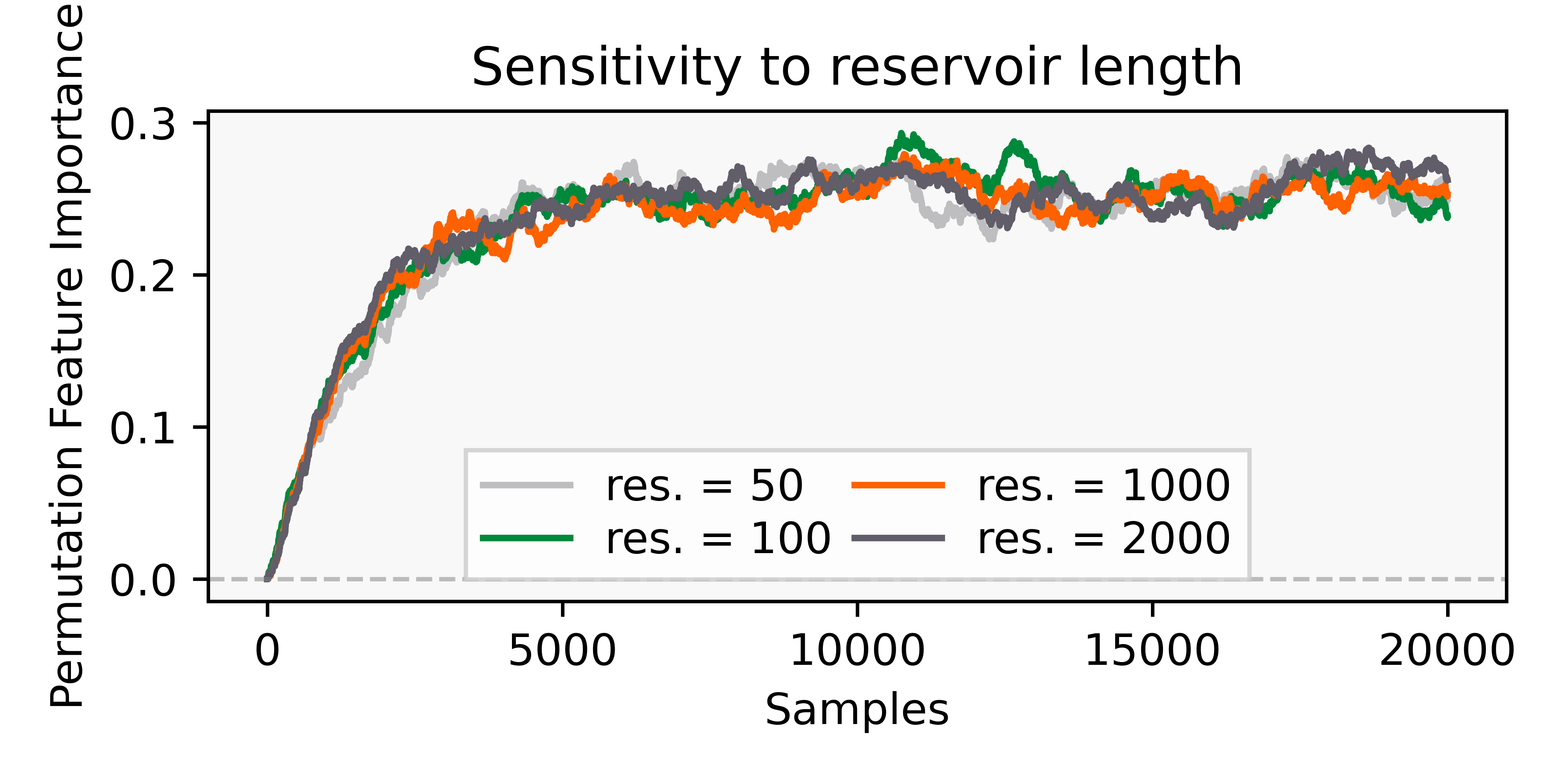}
    \end{minipage}
\end{minipage}
\caption{The importance of the \emph{nswprice} feature for an ARF model training on \emph{elec2} for different values of $\alpha$ (left) and reservoir length (right) iPFI parameters.}
\label{fig:parameter_analsis}
\end{figure*}

\subsection{Parameter Analysis}
In addition to the analysis of the sampling strategy, we further analyzed the effect of $\alpha$ smoothing parameter and the length of the reservoir used for sampling.
Figure~\ref{fig:parameter_analsis} shows how the $\alpha$ smoothing parameter has a substantial effect on the iPFI estimates.
The length of the reservoir has no  large effect on the estimates.
For both sensitivity analysis experiments, iPFI with geometric sampling is applied on an ARF and \emph{elec2}.
only the important \emph{nswprice} feature is plotted.
For each parameter value a single run is presented.

\begin{figure*}
    \centering
    \includegraphics[width=.82\textwidth]{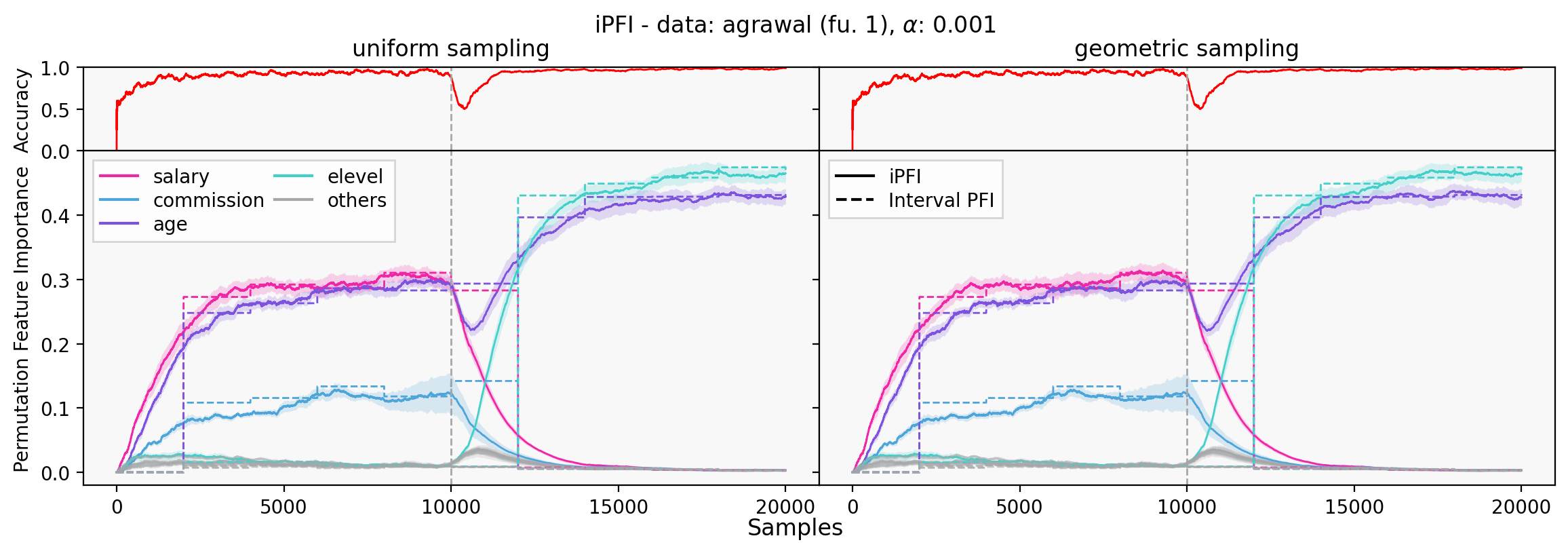}
    \includegraphics[width=.82\textwidth]{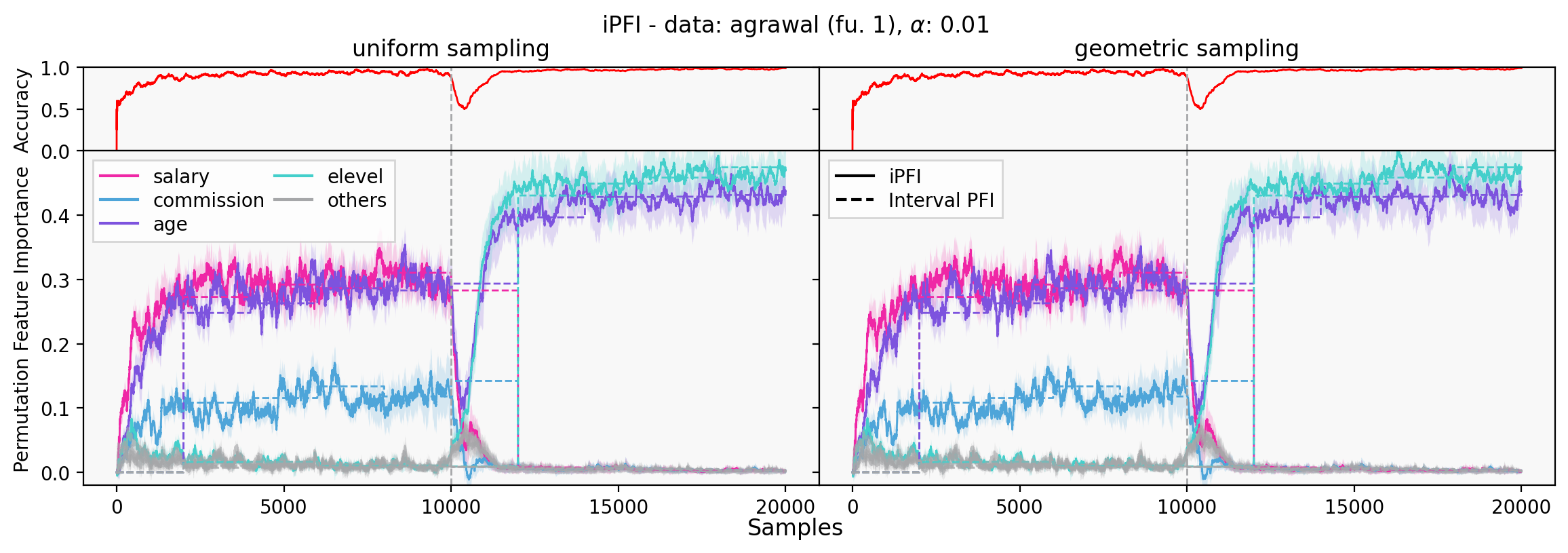}
    \caption{iPFI on \emph{agrawal} with a function-drift (fu. 1) after 10k samples with $\alpha = 0.001$ (top) and $\alpha = 0.01$ (bottom).}
    \label{fig:agrawal_fu. 1}
\end{figure*}

\begin{figure*}
    \centering
    \includegraphics[width=.82\textwidth]{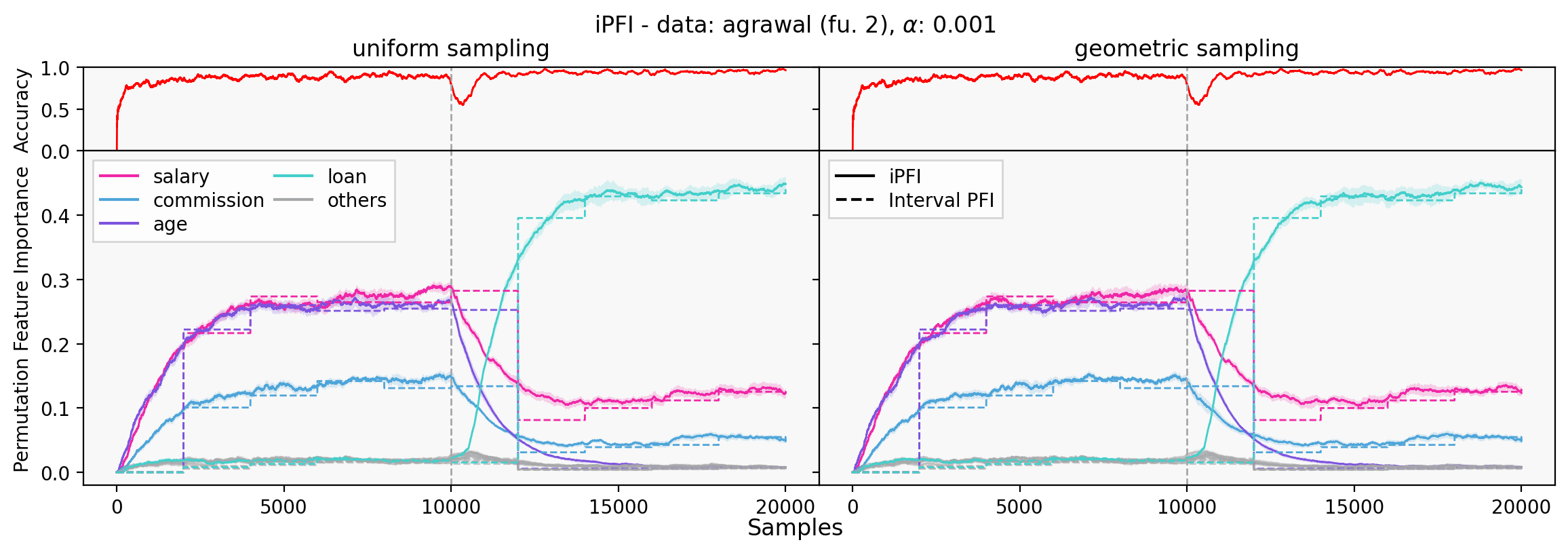}
    \includegraphics[width=.82\textwidth]{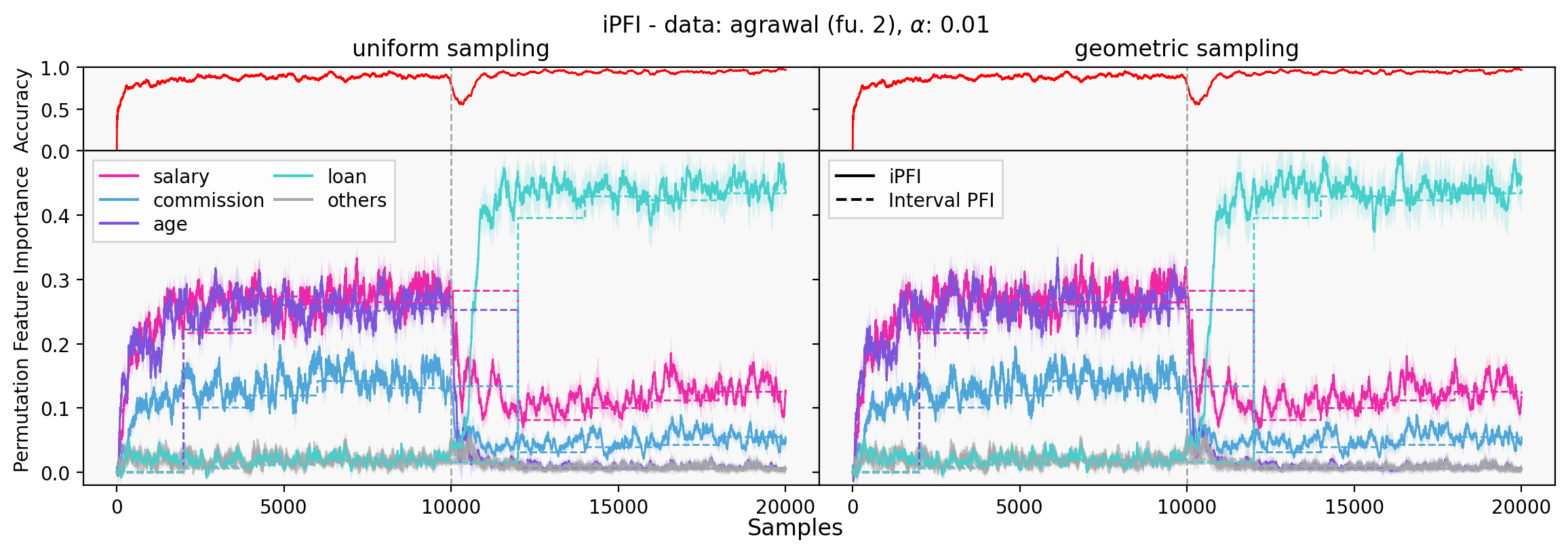}
    \caption{iPFI on \emph{agrawal} with a function-drift (fu. 2) after 10k samples with $\alpha = 0.001$ (top) and $\alpha = 0.01$ (bottom).}
    \label{fig:agrawal_fu. 2}
\end{figure*}

\begin{figure*}
    \centering
    \includegraphics[width=.82\textwidth]{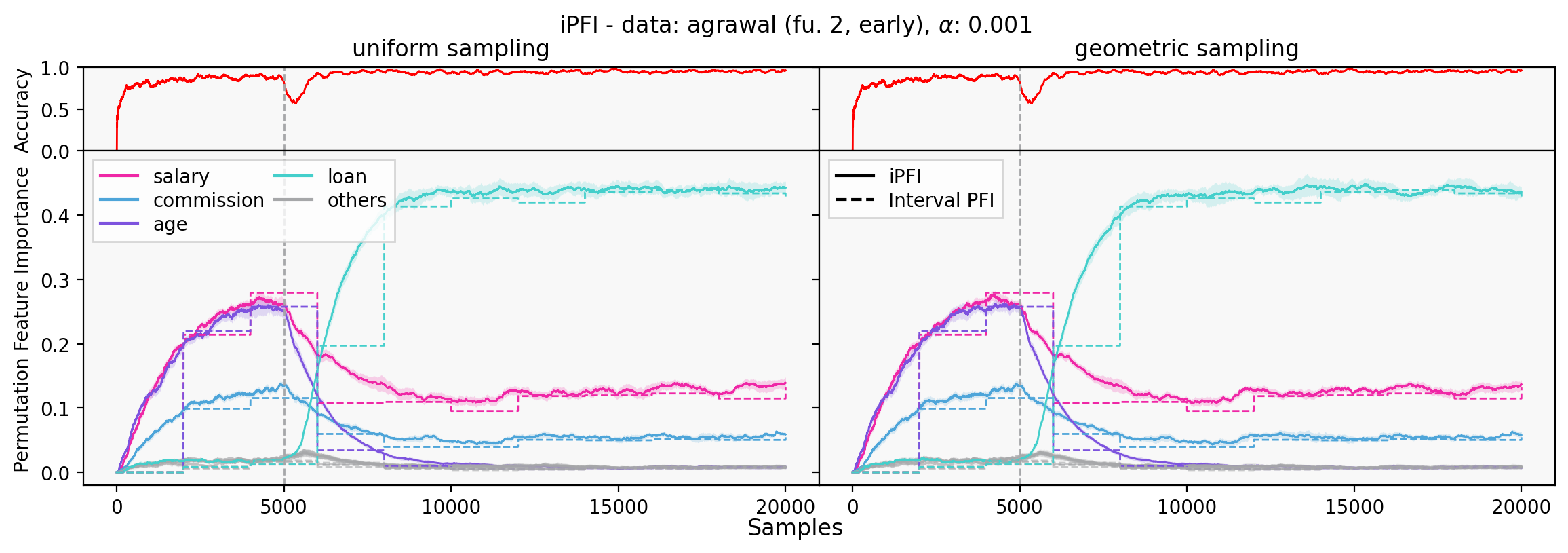}
    \includegraphics[width=.82\textwidth]{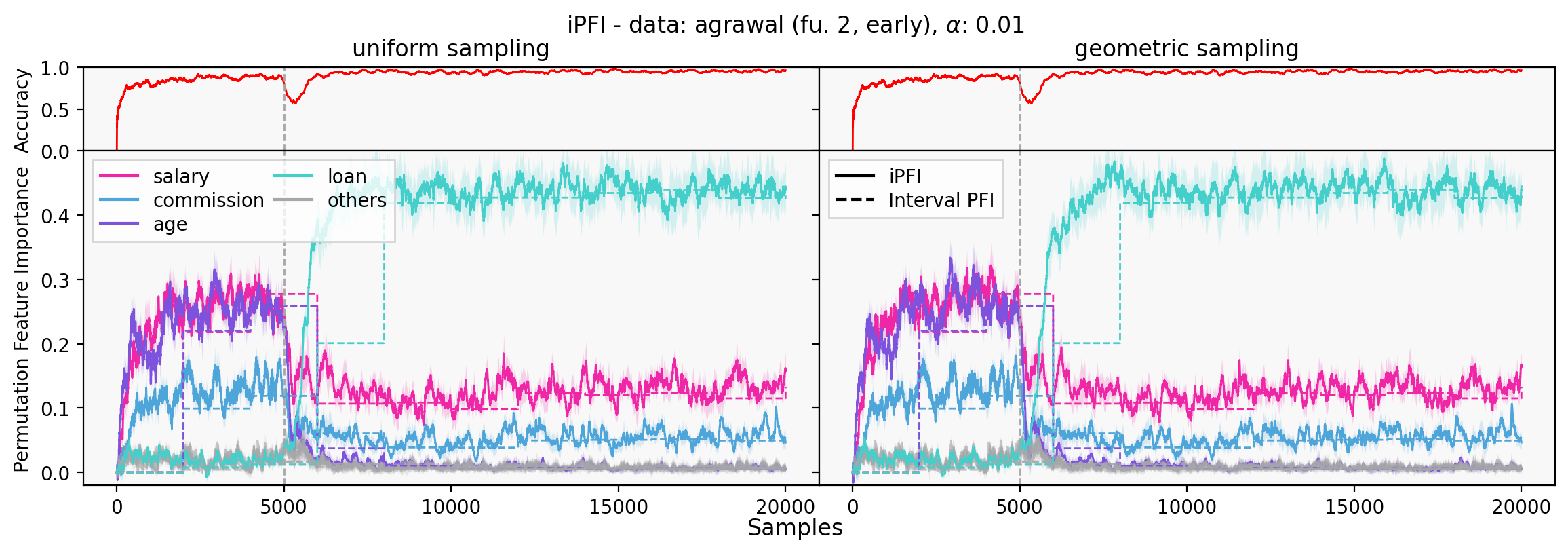}
    \caption{iPFI on \emph{agrawal} with a function-drift (fu. 2, early) after 5k samples with $\alpha = 0.001$ (top) and $\alpha = 0.01$ (bottom).}
    \label{fig:agrawal_fu. 2 early}
\end{figure*}

\begin{figure*}
    \centering
    \includegraphics[width=.82\textwidth]{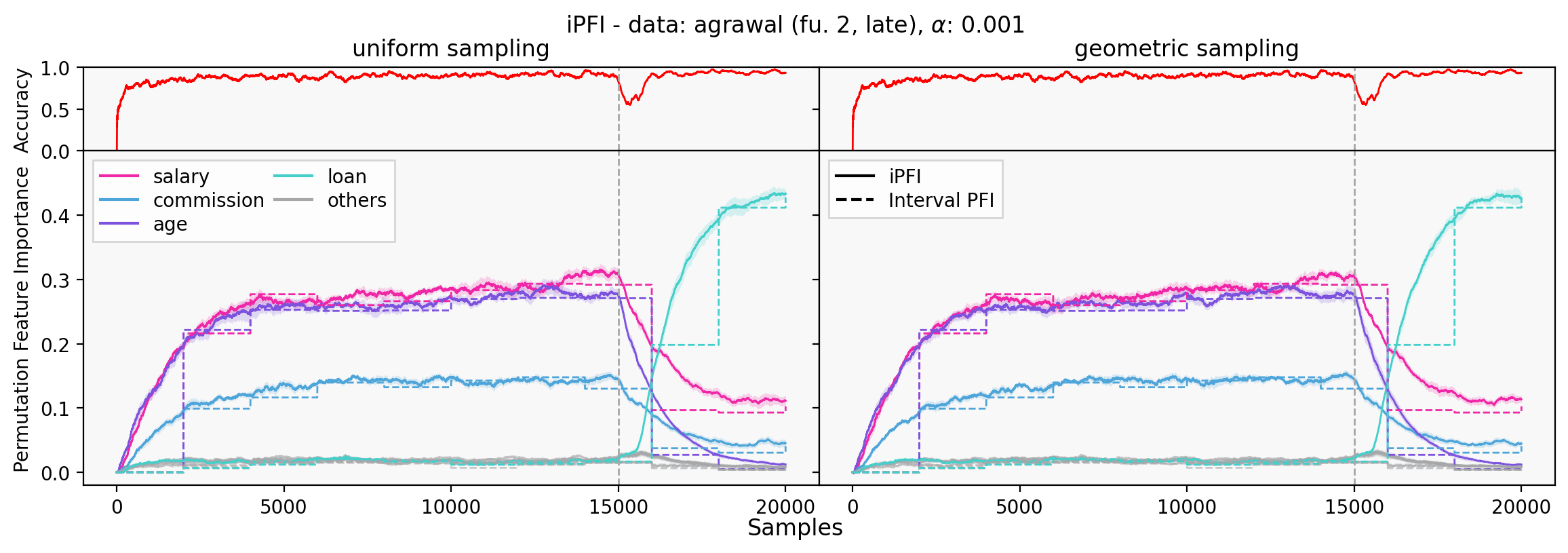}
    \includegraphics[width=.82\textwidth]{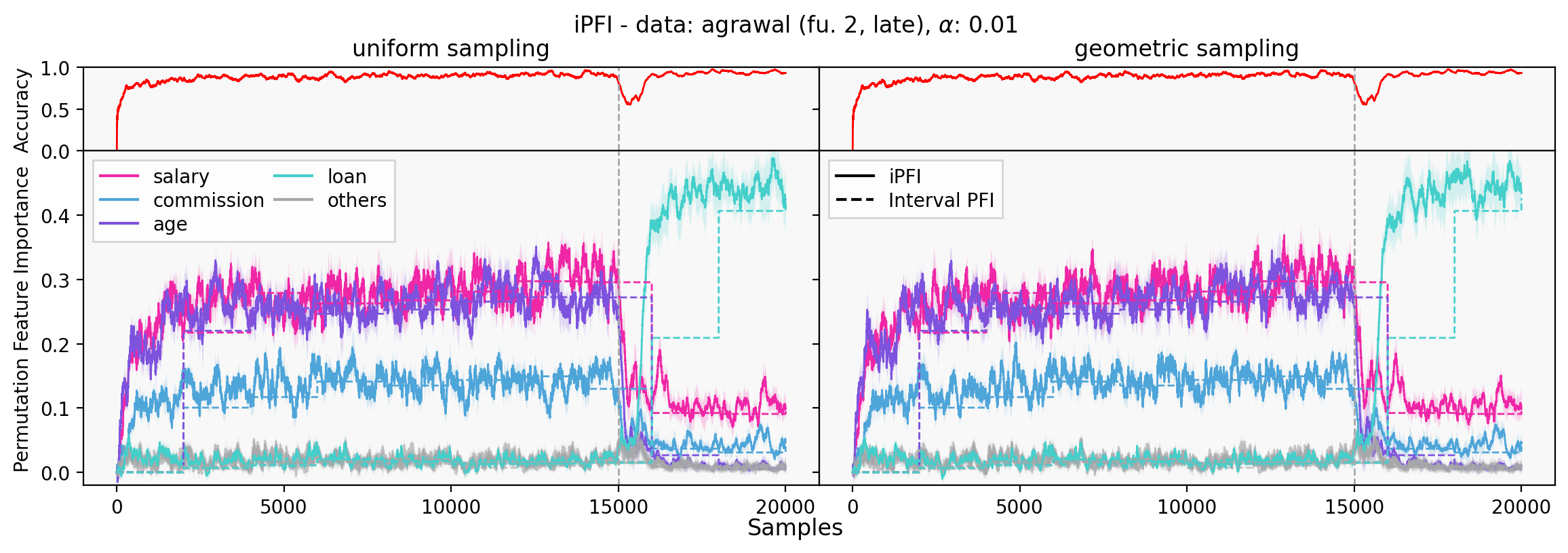}
    \caption{iPFI on \emph{agrawal} with a function-drift (fu. 2, late) after 15k samples with $\alpha = 0.001$ (top) and $\alpha = 0.01$ (bottom).}
    \label{fig:agrawal_fu. 2 late}
\end{figure*}

\begin{figure*}
    \centering
    \includegraphics[width=.82\textwidth]{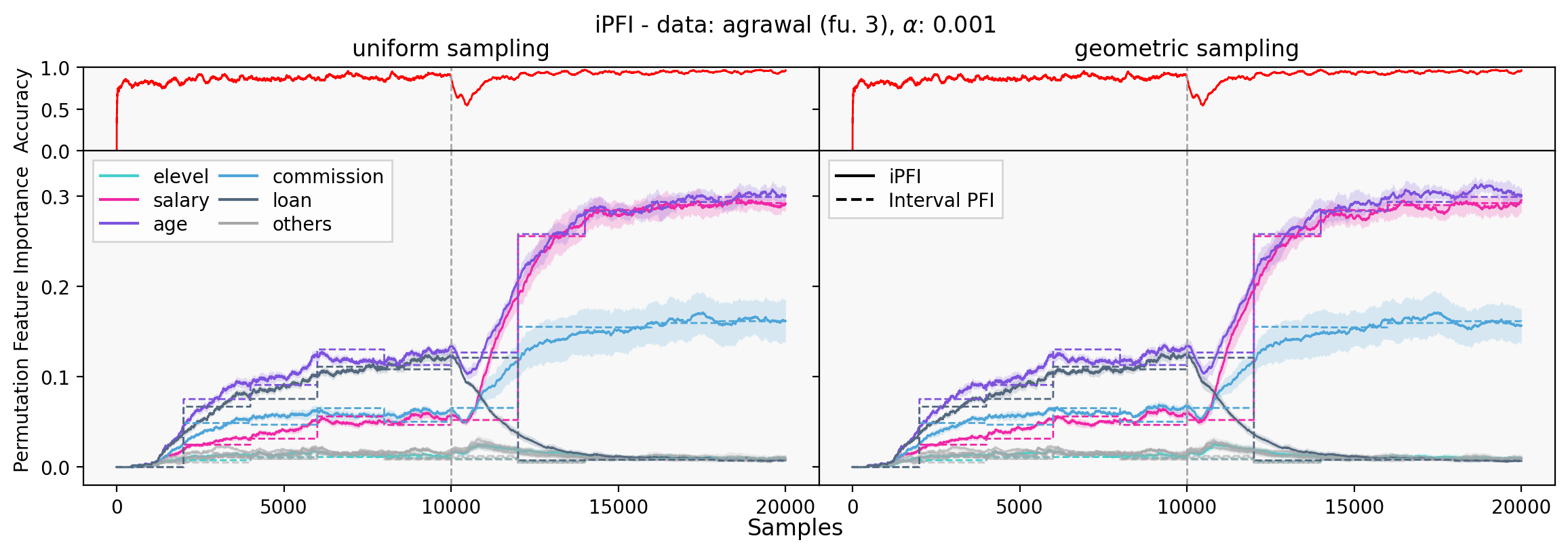}
    \includegraphics[width=.82\textwidth]{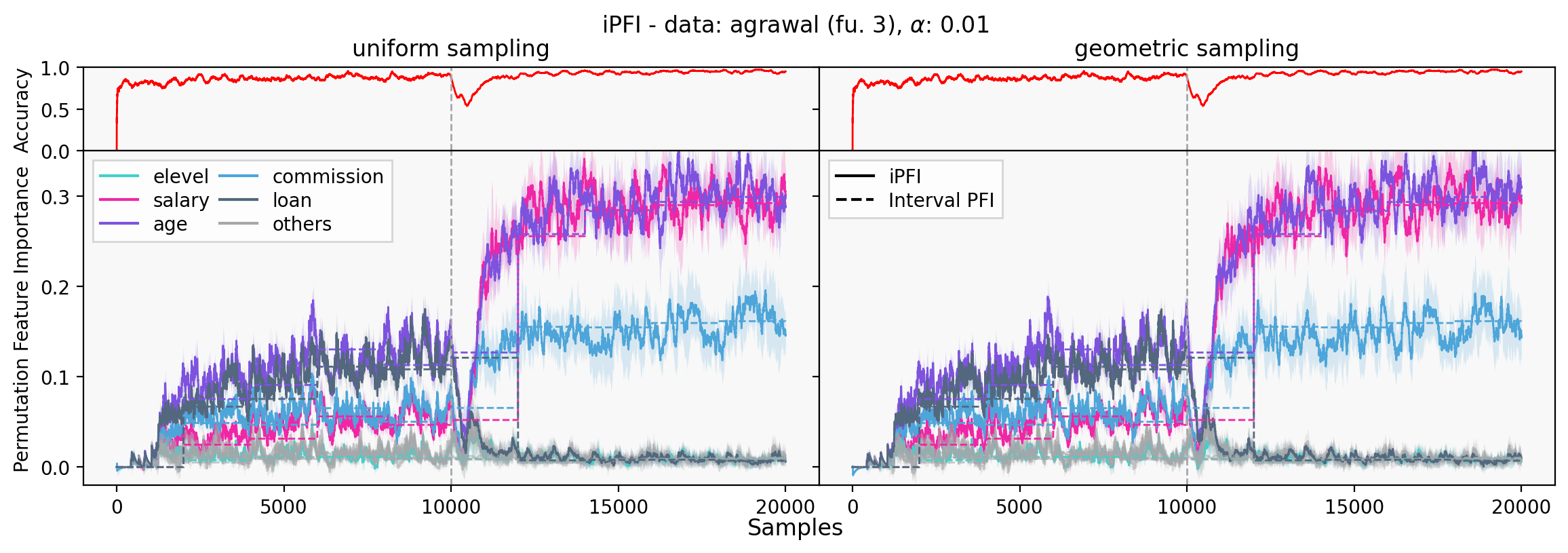}
    \caption{iPFI on \emph{agrawal} with a function-drift (fu. 3) after 10k samples with $\alpha = 0.001$ (top) and $\alpha = 0.01$ (bottom).}
    \label{fig:agrawal_fu. 3}
\end{figure*}

\begin{figure*}
    \centering
    \includegraphics[width=.8\textwidth]{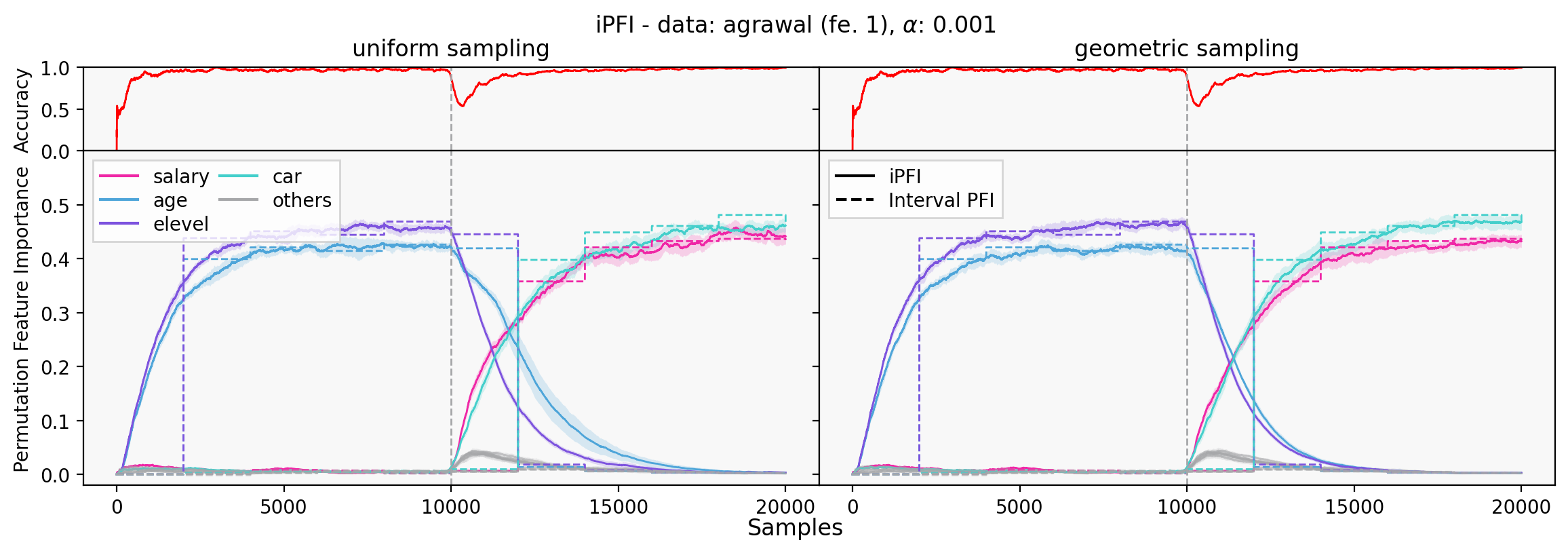}
    \includegraphics[width=.8\textwidth]{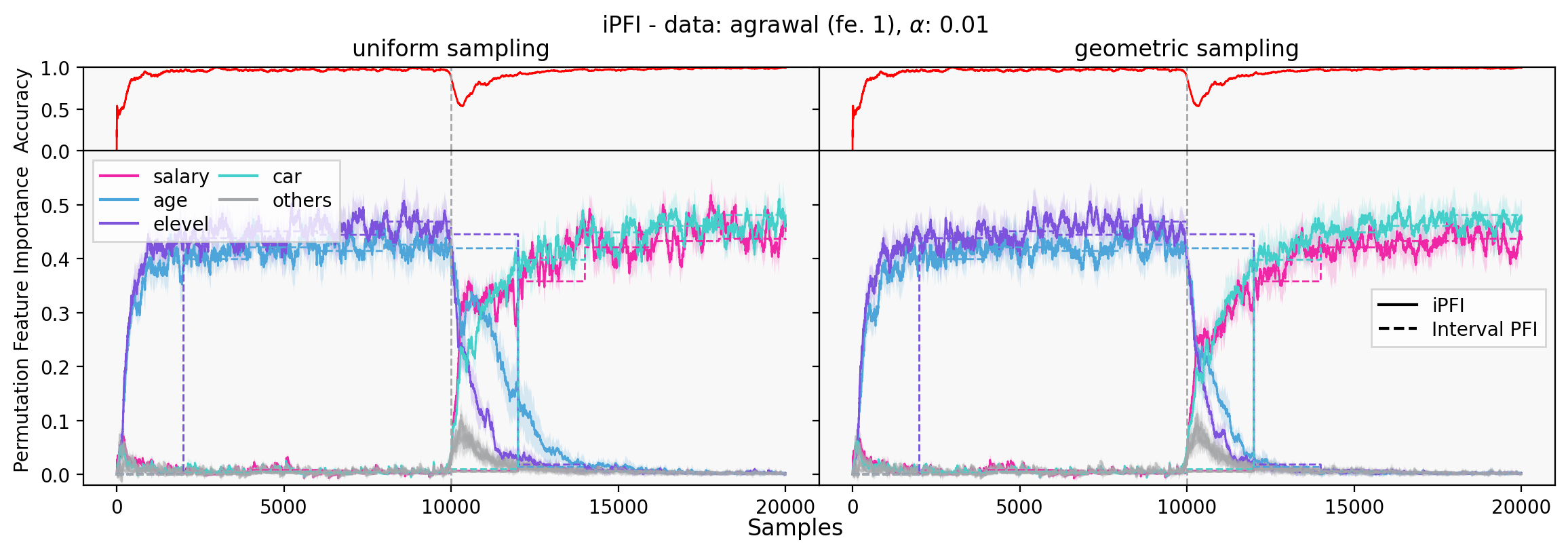}
    \caption{iPFI on \emph{agrawal} with a feature-drift (fe. 1) after 10k samples with $\alpha = 0.001$ (top) and $\alpha = 0.01$ (bottom).}
    \label{fig:agrawal_fe. 1}
\end{figure*}

\begin{figure*}
    \centering
    \includegraphics[width=.8\textwidth]{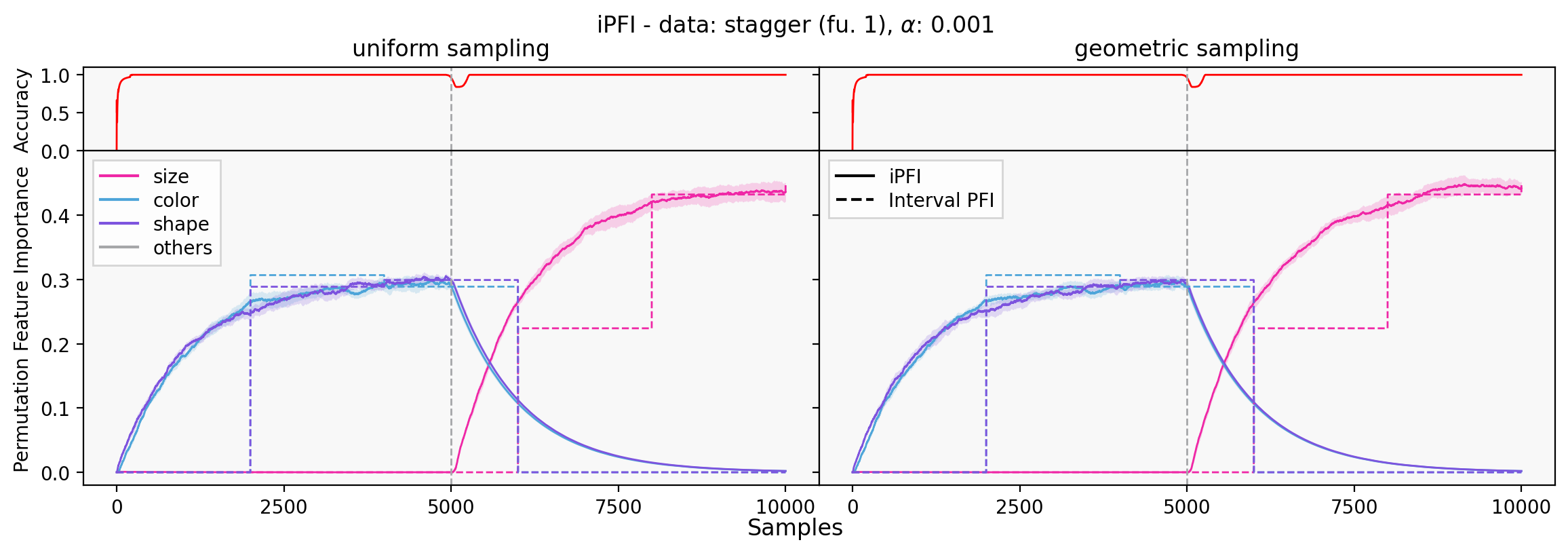}
    \caption{iPFI on \emph{stagger} with a feature-drift (fu. 1) after 5k samples with $\alpha = 0.001$.}
    \label{fig:stagger_fu. 1}
\end{figure*}

\begin{figure*}
    \centering
    \includegraphics[width=.8\textwidth]{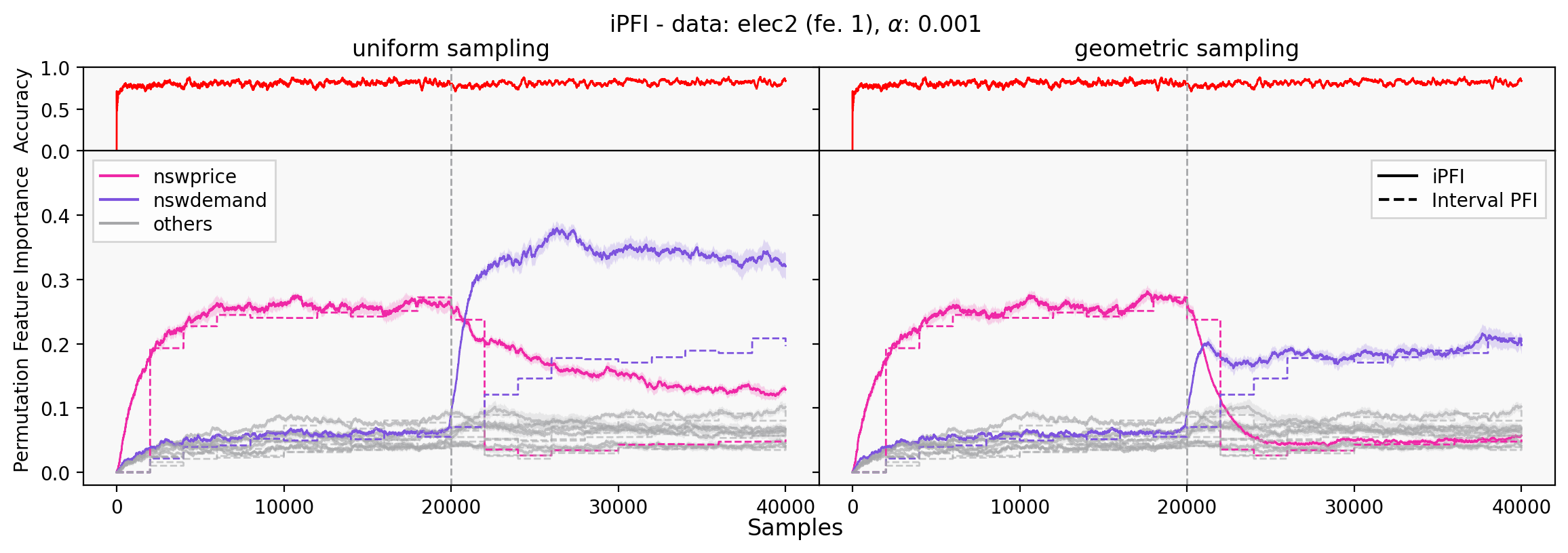}
    \includegraphics[width=.8\textwidth]{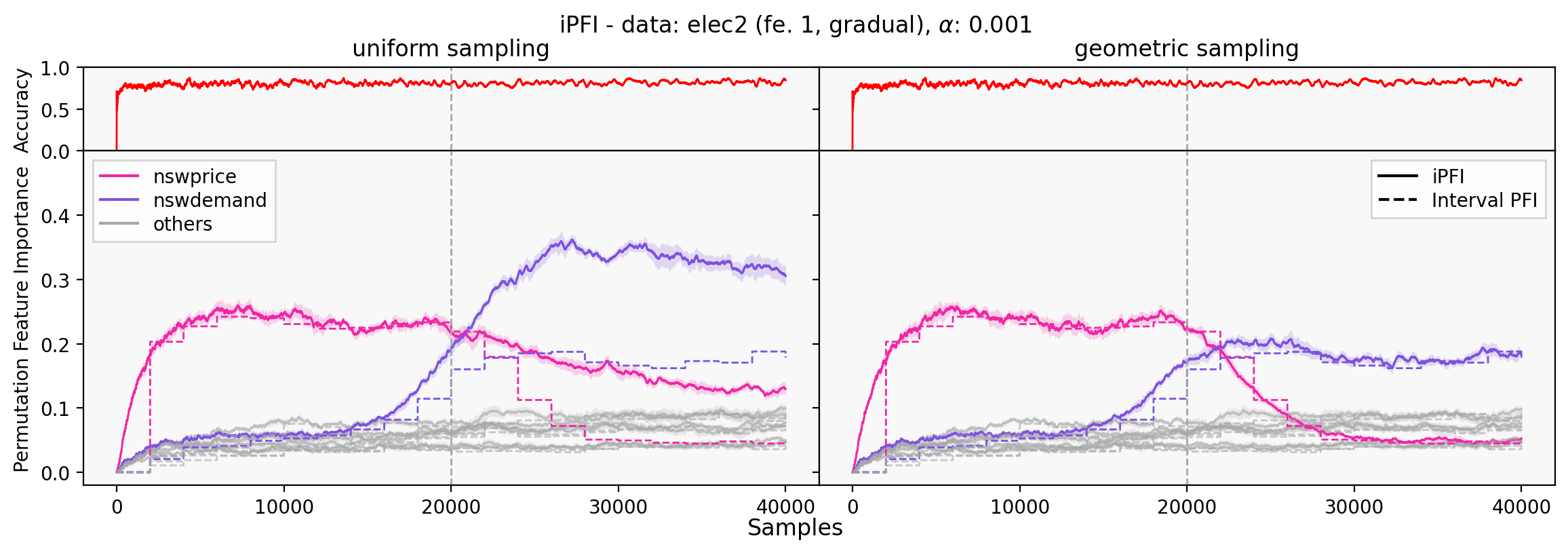}
    \caption{iPFI on \emph{elec2} with a sudden feature-drift (fe. 1) (top) and a gradual feature-drift (fe. 1, gradual) (bottom) after 20k samples with $\alpha = 0.001$.}
    \label{fig:elec2. fe. 1 and fe. 1, gradual}
\end{figure*}

\begin{figure*}
    \centering
    \includegraphics[width=.8\textwidth]{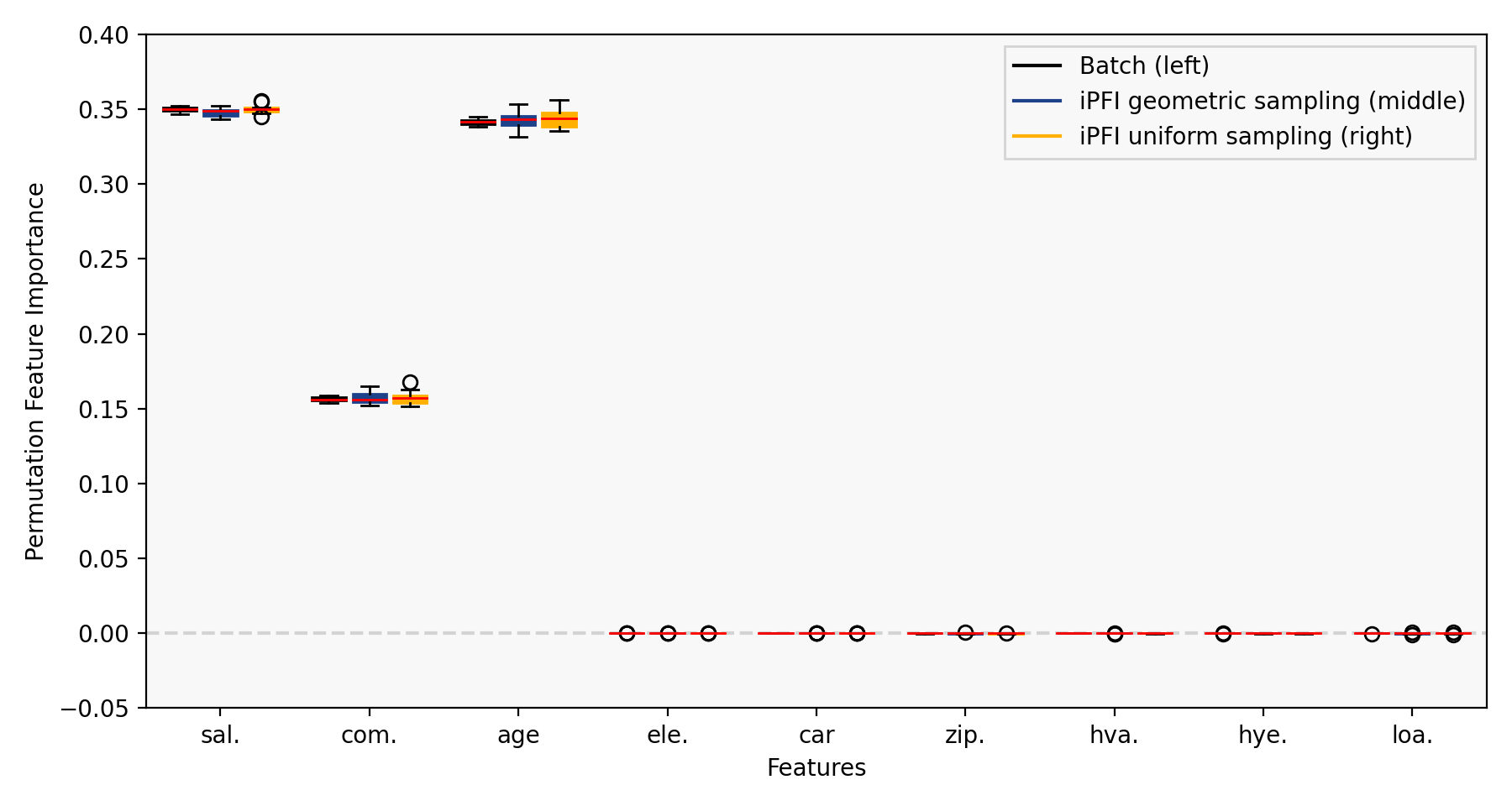}
    \caption{Boxplot of PFI estimates per feature of the \emph{agrawal} dataset for batch baseline (left), iPFI with geometric sampling (middle), and iPFI with uniform sampling (right) on a pre-trained static LGBM.}
    \label{fig:boxplot_adult}
\end{figure*}

\begin{figure*}
    \centering
    \includegraphics[width=.8\textwidth]{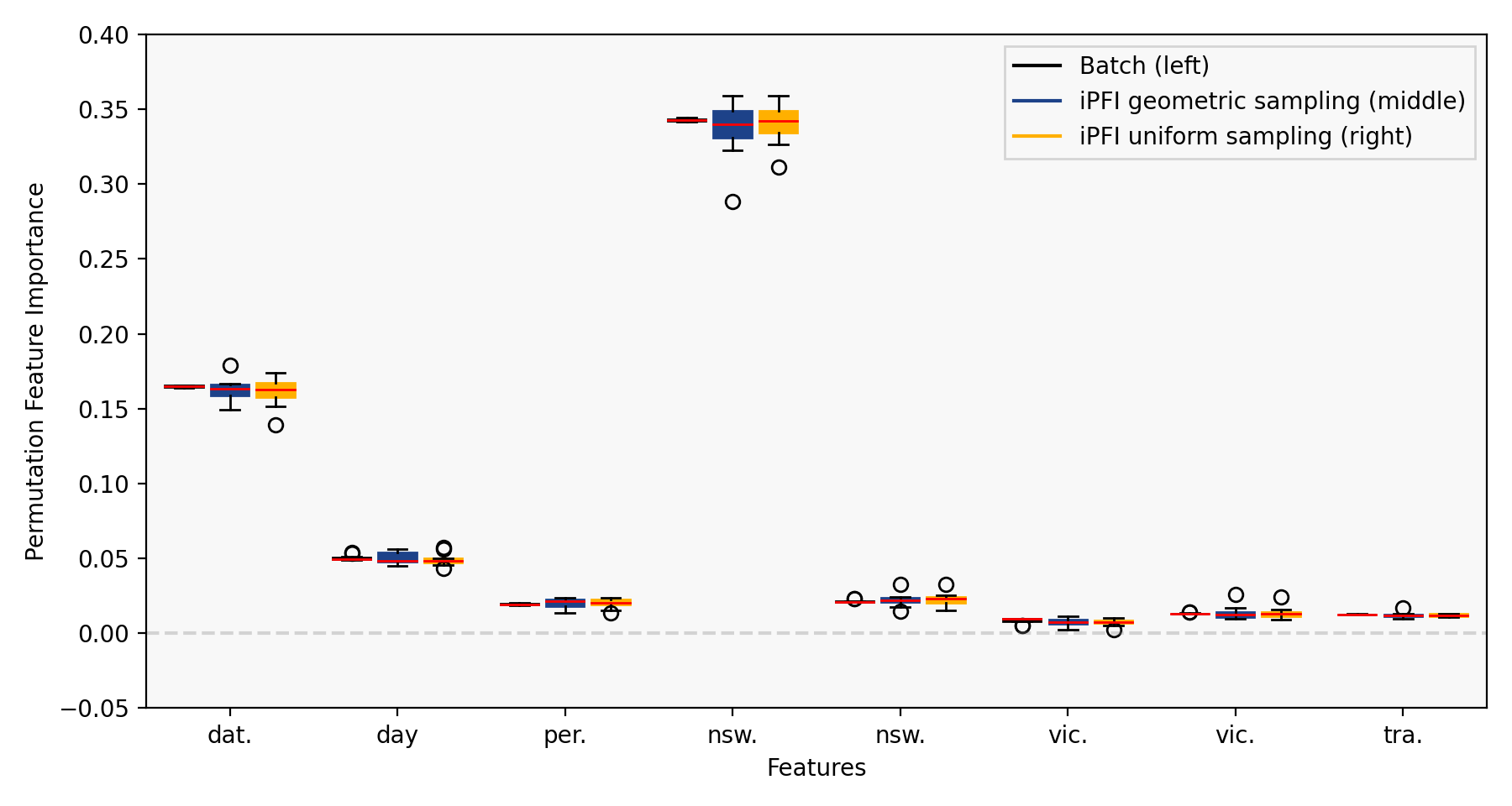}
    \caption{Boxplot of PFI estimates per feature of the \emph{elec2} dataset for batch baseline (left), iPFI with geometric sampling (middle), and iPFI with uniform sampling (right) on a pre-trained static LGBM.}
    \label{fig:boxplot_adult}
\end{figure*}

\begin{figure*}
    \centering
    \includegraphics[width=.8\textwidth]{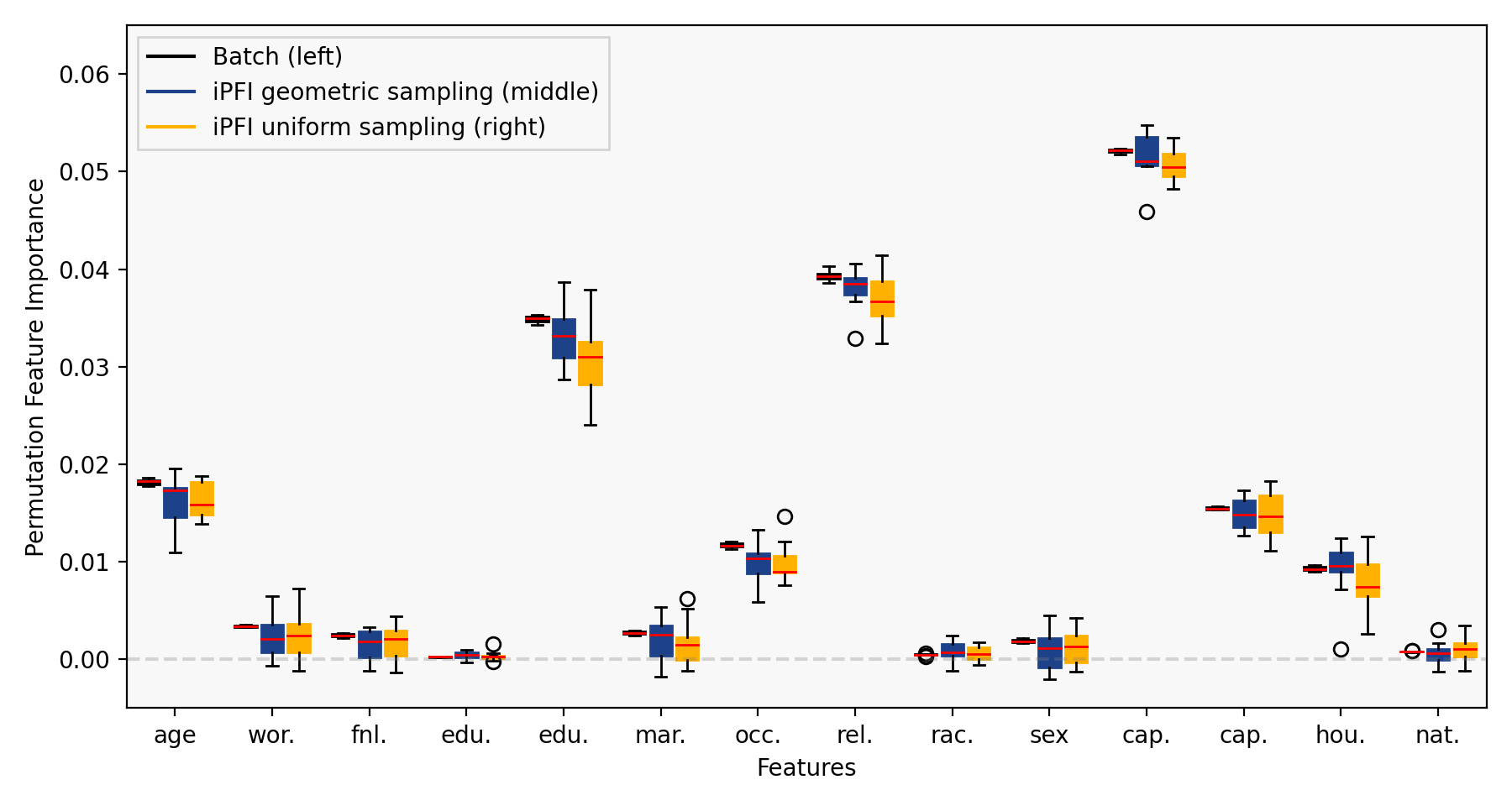}
    \caption{Boxplot of PFI estimates per feature of the \emph{adult} dataset for batch baseline (left), iPFI with geometric sampling (middle), and iPFI with uniform sampling (right) on a pre-trained static GBT.}
    \label{fig:boxplot_adult}
\end{figure*}

\begin{figure*}
    \centering
    \includegraphics[width=.8\textwidth]{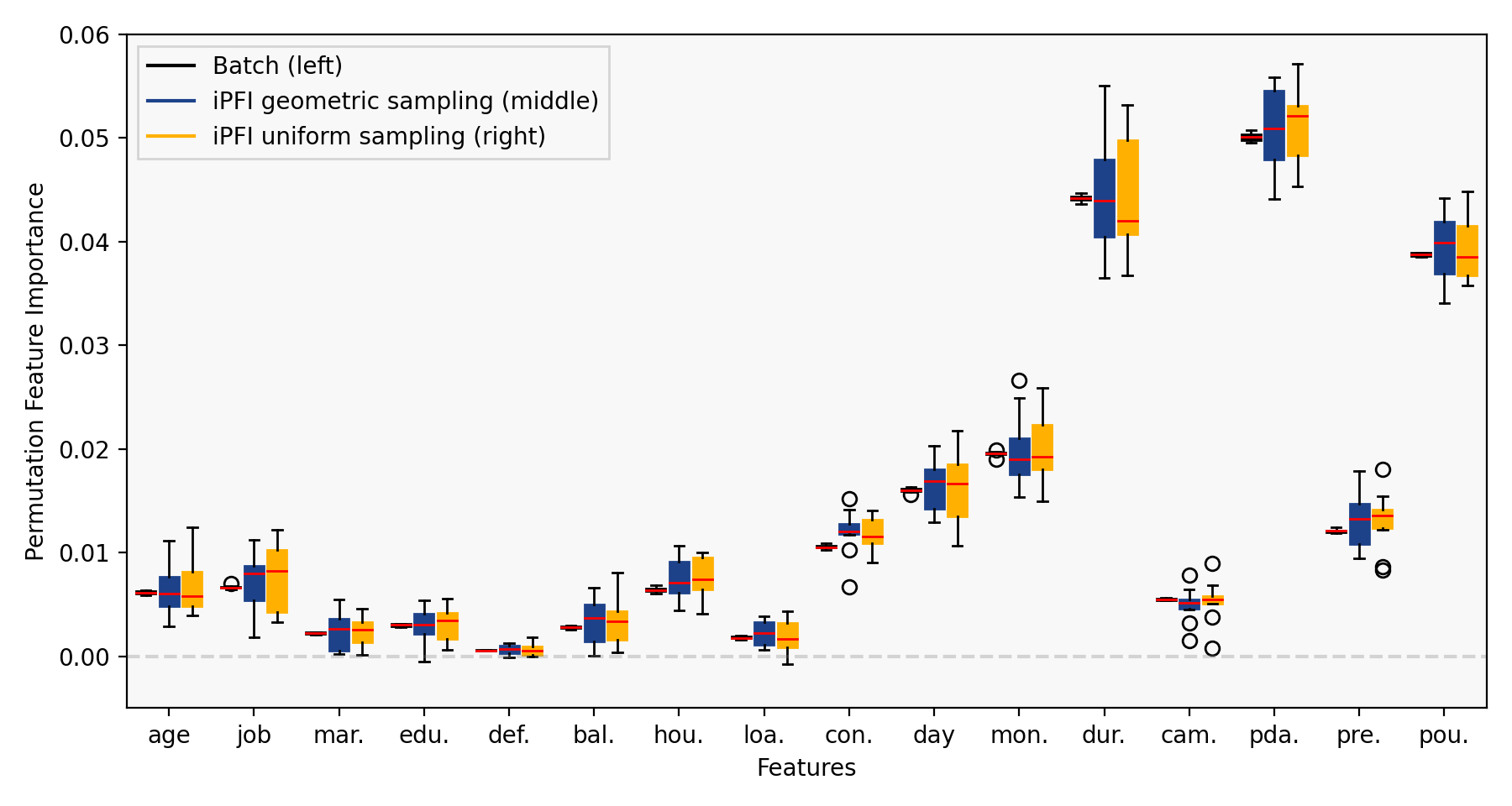}
    \caption{Boxplot of PFI estimates per feature of the \emph{bank} dataset for batch baseline (left), iPFI with geometric sampling (middle), and iPFI with uniform sampling (right) on a pre-trained static NN.}
    \label{fig:boxplot_adult}
\end{figure*}

\begin{figure*}
    \centering
    \includegraphics[width=.8\textwidth]{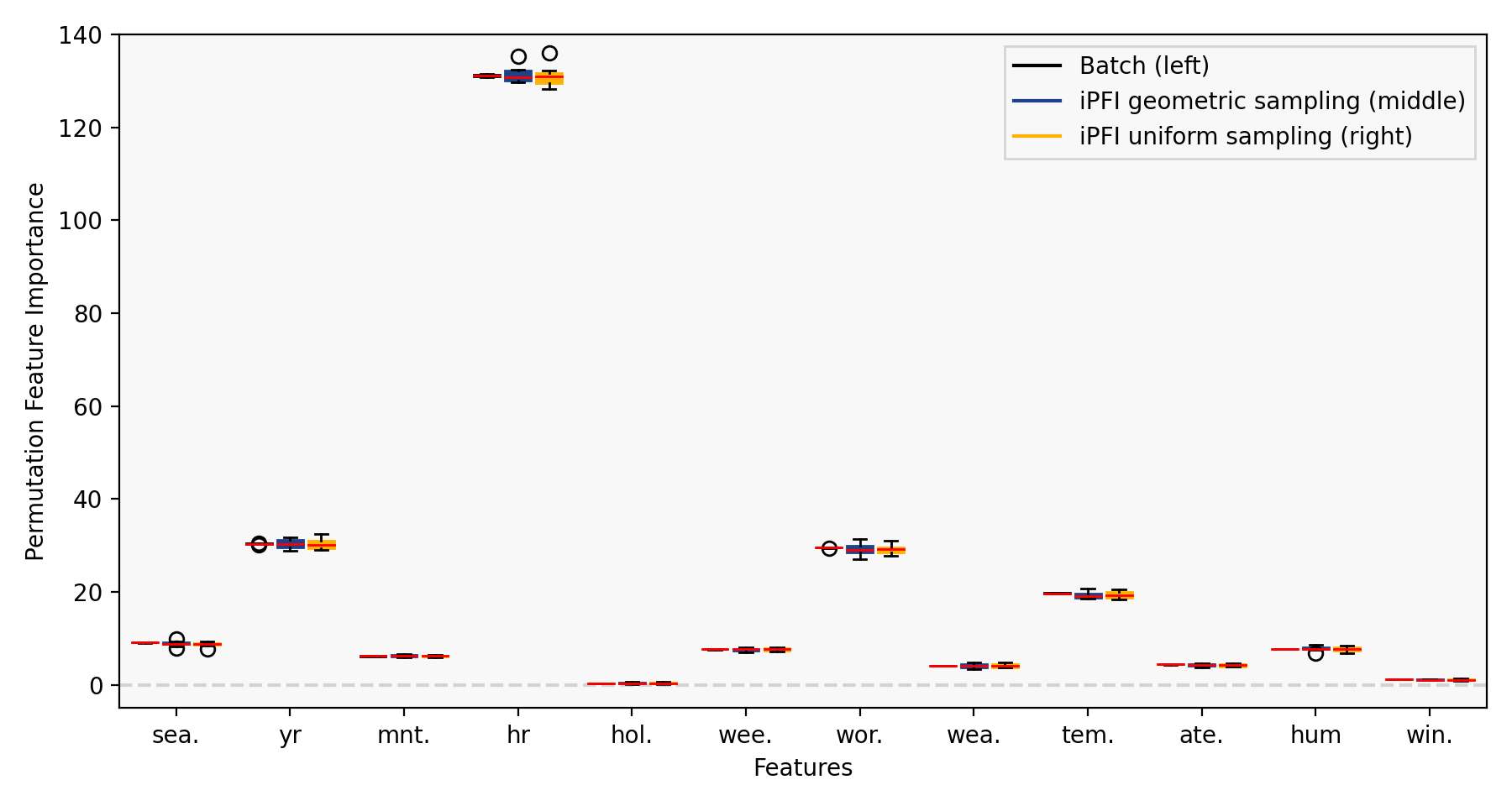}
    \caption{Boxplot of PFI estimates per feature of the \emph{bike} dataset for batch baseline (left), iPFI with geometric sampling (middle), and iPFI with uniform sampling (right) on a pre-trained static LGBM.}
    \label{fig:boxplot_adult}
\end{figure*}

\end{document}